\documentclass[10pt,twocolumn,letterpaper]{article}

\usepackage{cvpr}              
\usepackage{bm}
\usepackage{amsthm}
\newtheorem{theorem}{Theorem}
\definecolor{LightPink}{RGB}{255,220,220}
\definecolor{LightGreen}{RGB}{220,255,220}
\definecolor{DeepPink}{RGB}{255,150,150}
\definecolor{LimeGreen}{RGB}{120,200,120}
\usepackage{multirow}
\usepackage{adjustbox}
\usepackage{makecell}
\renewcommand{\paragraph}[1]{\vspace{1.25mm}\noindent\textbf{#1}}
\usepackage{graphicx}
\usepackage{caption}
\usepackage{subcaption}
\usepackage{setspace}

\usepackage{algorithm}
\usepackage{algpseudocode} 

\definecolor{cvprblue}{rgb}{0.21,0.49,0.74}
\usepackage[pagebackref,breaklinks,colorlinks,allcolors=cvprblue]{hyperref}

\def\paperID{18167} 
\def\confName{CVPR}
\def\confYear{2025}

\title{Training Data Provenance Verification: Did Your Model Use Synthetic Data from My Generative Model for Training?}

\author{
 Yuechen Xie\textsuperscript{\rm 1},
 Jie Song\textsuperscript{\rm 1}\thanks{Corresponding author},
 Huiqiong Wang\textsuperscript{\rm 2},
 Mingli Song\textsuperscript{\rm 1,3,4} \\[2mm]
 $^1$Zhejiang University \\
 $^2$Ningbo Innovation Center, Zhejiang University\\
 $^3$State Key Laboratory of Blockchain and Security, Zhejiang University \\
 $^4$Hangzhou High-Tech Zone (Binjiang) Institute of Blockchain and Data Security \\[2mm]
 {\tt\small \{xyuechen,sjie,huiqiong\_wang,brooksong\}@zju.edu.cn}
}

\begin{document}
\maketitle
\begin{abstract}
High-quality open-source text-to-image models have lowered the threshold for obtaining photorealistic images significantly, but also face potential risks of misuse. Specifically, suspects may use synthetic data generated by these generative models to train models for specific tasks without permission, when lacking real data resources especially. Protecting these generative models is crucial for the well-being of their owners. In this work, we propose the first method to this important yet unresolved issue, called Training data Provenance Verification (TrainProVe). The rationale behind TrainProVe is grounded in the principle of generalization error bound, which suggests that, for two models with the same task, if the distance between their training data distributions is smaller, their generalization ability will be closer. We validate the efficacy of TrainProVe across four text-to-image models (Stable Diffusion v1.4, latent consistency model, PixArt-$\alpha$, and Stable Cascade). The results show that TrainProVe achieves a verification accuracy of over 99\% in determining the provenance of suspicious model training data, surpassing all previous methods. Code is available at \url{https://github.com/xieyc99/TrainProVe}.
\end{abstract}    
\section{Introduction}
\label{sec:intro}

High-quality open-source text-to-image models based on diffusion models~\cite{ho2020denoising} can generate photorealistic images with simple text prompts~\cite{chenpixart,luo2023latent,perniaswurstchen,rombach2022high}, making them highly significant in terms of lowering creative barriers, enabling personalized customization, etc. Currently, most open-source text-to-image models are only permitted for academic or educational purposes, and commercial use is prohibited without explicit permission~\cite{chenpixart,luo2023latent,perniaswurstchen,rombach2022high,chen2024pixart,li2024blip,patil2024amused,ramesh2022hierarchical}. Therefore, protecting the intellectual property of these generative models is crucial for safeguarding the interests of their owners.

\begin{figure}[t]
  \centering
   \includegraphics[width=\linewidth, trim=1.35cm 0.9cm 0.9cm 1.5cm, clip]{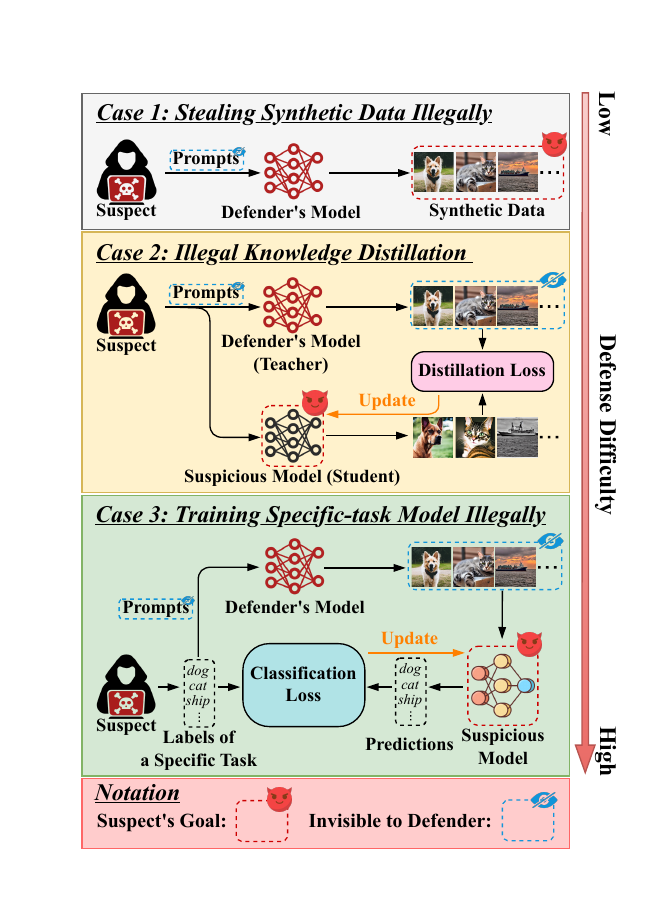}
   \vspace{-1.5em}
   \caption{The overview of the three cases. These three cases encompass nearly every conceivable real-world scenario, where a suspect uses synthetic images from the defender's text-to-image model illegally. In this paper, we are committed to addressing the security risks in \textit{Case 3}.}
   \vspace{-1.5em}
   \label{fig:cases}
\end{figure}

Currently, one of the major security challenges faced by defenders, \ie, the generative model owners, is the unauthorized use of the synthetic data from their models by suspects~\cite{wang2023detect,wang2024did,zhang2024training,liu2023watermarking}. As shown in \cref{fig:cases}, based on the suspect's intention, this issue can be divided into three cases:

\textit{Case 1}: the suspect steals the synthetic data illegally and claims it as his own personal work~\cite{wang2024did,liu2023watermarking,tancik2020stegastamp}, implying the intellectual property of the generative model is infringed. 

\textit{Case 2}: the suspect steals the synthetic data from the defender's generative model, and trains another generative model for the same task using knowledge distillation techniques~\cite{salimansprogressive,hinton2015distilling,xiang2024dkdm,dockhorn2023distilling,mekonnen2024adv}. In this case, the defender cannot access the synthetic data used by the suspect, making it a more challenging defense task compared to \textit{Case 1}. 

\textit{Case 3}: the suspect steals the synthetic data from the generative model to train another model for a specific task~\cite{yuan2022not,sariyildiz2023fake,he2022synthetic,shipard2023diversity,zhou2023training,bansal2023leaving}. In this case, the defender is not only unaware of the training data used by the suspicious model, but also cannot compare the similarity between the generative model and the suspicious model directly due to the different tasks, representing an even more arduous defense task compared to \textit{Case 2}.

For \textit{Case 1} and \textit{Case 2}, there are already some mature methods to address these issues, such as adding watermarks~\cite{zhao2023recipe,liu2023watermarking,xiong2023flexible,yuan2024watermarking,min2024watermark} and training data attribution~\cite{zhang2024training}, which identify illegal activities by checking the distributions of the suspicious images or the synthetic images from suspicious model. However, for \textit{Case 3}, although the potential of using synthetic data to train models for other tasks has been recognized gradually in recent years~\cite{yuan2022not,sariyildiz2023fake,he2022synthetic,shipard2023diversity} due to the lower collection costs compared to real data, the associated security risks have been overlooked. Methods that work for the first two cases cannot be applied to the third case, because the learning task of the suspicious model is not necessarily image generation. For example, for a suspicious image classification model, there are no synthetic data available for defenders to examine.

In this work, we address this important yet unexplored issue. To the best of our knowledge, we propose the first method for protecting the intellectual property of generative models in such scenario, which we call Training data Provenance Verification (TrainProVe). It helps defenders verify whether the suspicious model has been trained using synthetic data from their generative model. Given that suspicious models may use synthetic data for training without authorization, and is used for commercial purposes, we focus on the black-box setting, where the defenders have no information about the training configuration of the suspicious model (\eg, loss function, model architecture, etc.) and can only access the model through Machine Learning as a Service (MLaaS)~\cite{ribeiro2015mlaas,sun2015mlaas,weng2022mlaas}. This means that the defenders can only obtain prediction results (such as prediction logits or labels) via the suspicious model's API.

The rationale behind TrainProVe is grounded in the principle of generalization error bound, which suggests that, for two models with the same task, if the distance between their training data distributions is smaller, their generalization ability will be closer. Specifically, defenders can train a model for the same task as the suspicious model, using the synthetic data generated by their generative model. By comparing the performances of these two models, defenders can determine whether the suspicious model was trained on synthetic data from their generative model. More concretely, 
the proposed TrainProVe involves three steps: (1) using the defender's generative model along with different text prompts to generate a \textit{shadow dataset} and a \textit{validation dataset}, based on the suspicious model's task; (2) training a \textit{shadow model} for the same task as the suspicious model using the shadow dataset; (3) performing hypothesis testing on the prediction accuracy of the suspicious model and the shadow model on the validation dataset, in order to determine whether the suspicious model was trained using synthetic data from the defender's generative model.

In summary, the principal contributions of this paper are threefold: (1) we introduce the important yet under-explored issue of intellectual property protection for text-to-image models, in which the suspect uses the synthetic data from the defender's generative model illegally to train a specific-task suspicious model, such as an image classification model. For the defender, the suspicious model is a black-box; (2) we discover and demonstrate that the similarity between different models is related to their source domains closely, and based on this, we propose the training data provenance verification technique. To the best of our knowledge, this is the first method to address such intellectual property protection issue; (3) Extensive experiments show that TrainProVe achieves over 99\% accuracy in verifying training data provenance for suspicious models, surpassing all previous studies.
\section{Related Work}
\label{sec:related}

\paragraph{Text-to-image diffusion model.}
Diffusion models~\cite{ho2020denoising} have revolutionized generative modeling fundamentally, in terms of quality and generalization. Their excellent scalability has led to significant advancements in specific applications, such as text-to-image generation. Text-to-image diffusion models~\cite{chenpixart,perniaswurstchen,rombach2022high} use the diffusion process to transform text descriptions into realistic images. Recently, people have begun to recognize the potential of text-to-image diffusion models for synthesizing dataset~\cite{sariyildiz2023fake,shipard2023diversity,zhou2023training,jing2021turning,jing2023deep}. However, so far, no one has realized that suspects might use these text-to-image diffusion models to synthesize datasets and train their own models without authorization. Therefore, in this paper, we propose the first feasible solution to address this issue, securing and fostering healthy development in this field.

\paragraph{Text-to-image model protection.}
Training generative models requires large amounts of data and computational resources~\cite{rombach2022high,chenpixart}, making the intellectual property of these models crucial for their owners. When these models are open-source~\cite{chenpixart,luo2023latent,perniaswurstchen,rombach2022high}, their intellectual property is vulnerable to infringement by suspects. Specifically, for text-to-image models, suspects have the following three ways primarily to infringe on their intellectual property: (1) stealing synthetic data illegally; (2) illegal knowledge distillation; (3) training suspicious models for specific tasks using synthetic data illegally. 
For the first two cases, watermarking~\cite{liu2023watermarking,luo2023steal,zhao2023recipe} and training data attribution~\cite{zhang2024training} have been developed to address these issues. However, for the third case, no suitable solution exists currently.
This work aims to address this gap.

\paragraph{Generalization error bound.}
Generalization error (upper) bound is used to describe a model's performance on testing data and widely applied in domain adaptation~\cite{wang2018deep,farahani2021brief}, which is a special case of transfer learning~\cite{torrey2010transfer,zhuang2020comprehensive}.
Theorem 1 in~\cite{ben2006analysis} states the generalization error upper bound of model \( M \) on the target domain $T$. Let \( \mathcal{H} \) be a hypothesis space with VC dimension \( d \), and let \( m \) be the sample size of the dataset in the source domain $S$. Then, with probability at least \( 1 - \delta \), for every \( M \in \mathcal{H} \):
\begin{equation}
    \epsilon_T(M) \leq \hat{\epsilon}_S(M) + \sqrt{\frac{4}{m}\left(d\log\frac{2em}{d}\right)}+d_{\mathcal{H}}(P_S,P_T)+\lambda,
\label{eq}
\end{equation}
where \( \epsilon_T(M) \) and \( \hat{\epsilon_S}(M) \) represent the generalization error of $M$ on the target domain and the empirical error of $M$ on the source domain, respectively. \( P_S \) and \( P_T \) denote the marginal probability distributions of the source and target domains, respectively. \( d_\mathcal{H}(P_S,P_T) \) is the \( \mathcal{H} \)-divergence between the source and target domains, which measures the distance between the two distributions within the hypothesis space \( \mathcal{H} \). \(\lambda\) is a constant and $e$ is the base of the natural logarithm.
The rationale behind TrainProVe is grounded in this theorem, which suggests that, for two models with the same task, if the distance between their training data distributions is smaller, their generalization ability will be closer.

\section{Method}
\label{sec:method}

\subsection{Problem Formulation}

In this study, we focus on the training data provenance verification task in image classification. This task involves two roles: the \textit{defender} and the \textit{suspect}. The defender, embodying the role of text-to-image model provider, endeavors to ascertain whether the black-box suspicious model, ${M}_{sus}$, has been illicitly trained on synthetic dataset generated by his model $G_d$ (as shown in \textit{Case 3} of \cref{fig:cases}). More specifically, for a black-box suspect model, the defender's objective is to determine whether it is illegal or legal based solely on the labels it outputs.

${M}_{sus}$ can be classified as either illegal or legal based on its training datasets:
(1) \textit{Illegal}: ${M}_{sus}$ is trained on the synthetic data from the defender's model $G_d$, indicating the occurrence of intellectual property infringement;
(2) \textit{Legal}: ${M}_{sus}$ is trained using a different data source independent of the defender, such as some real data or synthetic data from other generative models, implying the innocence of the suspect. 
To be precise, when presented with a black-box suspicious model, the defender's objective is to determine whether it is an illegal or legal model.

\begin{figure*}[t]
  \centering
   \includegraphics[width=\linewidth, trim=1.5cm 0.4cm 0.8cm 0.2cm, clip]{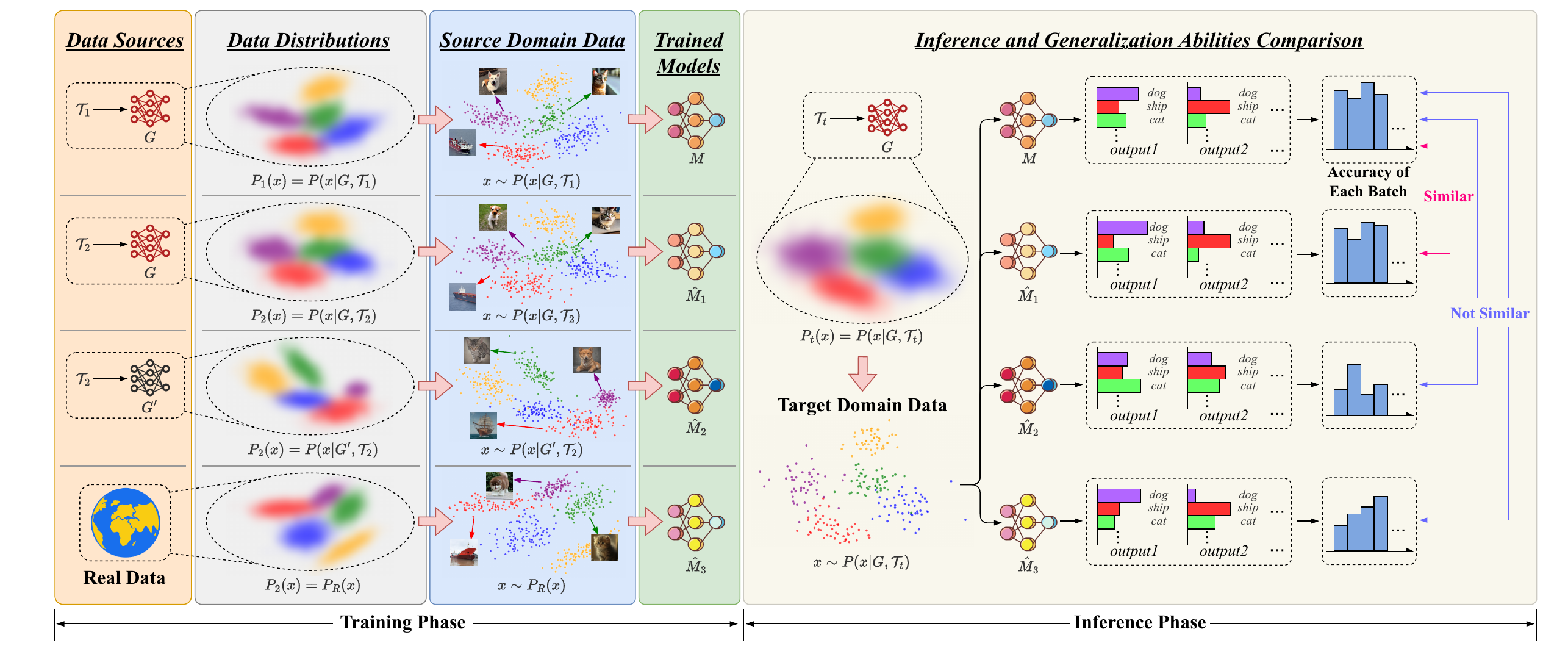}
   \vspace{-1.0em}
   \caption{The overview of TrainProVe's motivation. $M$ is trained on the dataset whose data distribution is $P(x|G,\mathcal{T}_1)$. $\hat{M}_{1}$, $\hat{M}_{2}$, and $\hat{M}_{3}$ are $\hat{M}$ trained on three datasets with different data distributions, respectively. The source and target domain data are sampled from their respective data distributions. \( G \) and \( G' \) are two different text-to-image models, and \( \mathcal{T}_1 \), \( \mathcal{T}_2 \), and \( \mathcal{T}_t \) are distinct sets of text prompts.}
   \vspace{-1em}
   \label{fig:motivation}
\end{figure*}

\subsection{The Theoretical Insight}
The underlying rationale for TrainProVe is that when two models are assigned the same task, their ability to generalize will be more aligned if the divergence between their training data distributions is diminished. We demonstrate this through a case study of $K$-way classification, where \( \mathcal{X}= \mathbb{R}^{C\times W \times H} \) and \( \mathcal{Y} = \{1,\ldots, K\} \) represent the input and label spaces respectively.
We follow the setting of previous domain adaptation works~\cite{wang2018deep,pan2009survey}, viewing a domain as composed of a feature space \( \mathcal{X} \) and a marginal probability distribution \( P(\bm{x}) \), \ie, \( \{ \mathcal{X}, P(\bm{x}) \} \), where \(\bm{x} \in \mathcal{X}\). 

Let \( S_1=\{ \mathcal{X}, P_1(\bm{x})\} \) and \( S_2=\{ \mathcal{X}, P_2(\bm{x}) \} \) be two source domains sharing the same feature space but possessing 
distinct marginal probability distributions, \ie, \( P_{1}(\bm{x}) \neq P_{2}(\bm{x}) \). \( \mathcal{D}_{1} = \{(\bm{x}^i_{1},y^i_{1}) \}_{i=1}^{|\mathcal{D}_{1}|} \) and \( \mathcal{D}_{2} = \{(\bm{x}^j_{2},y^j_{2})\}_{j=1}^{|\mathcal{D}_{2}|} \) represent the
datasets sampled from the source domains \( S_1 \) and \( S_2 \) respectively, where \(\bm{x}^i_{1} \sim P_{1}(\bm{x})\) and \(\bm{x}^j_{2} \sim P_{2}(\bm{x})\). \( M \) and \( \hat{M} \) are the models trained on \( \mathcal{D}_{1} \) and \( \mathcal{D}_{2} \) respectively. Then we have the following theorem:
\begin{theorem}
Assume \( M \) is trained on synthetic data generated by a text-to-image model \( G \) with a set of text prompts \( \mathcal{T}_1 \)
, \ie, \( P_1(\bm{x}) = P(\bm{x} | G, \mathcal{T}_1) \). \( \hat{M} \) can be trained on either real data or synthetic data generated by any text-to-image model. Based on the generalization errors of \( M \) and \( \hat{M} \) on the target domain \( T=\{ \mathcal{X}, P(\bm{x}| G, \mathcal{T}_t) \} \), we have:
\begin{equation}
    \sup_{P_2(\bm{x}) = P(\bm{x} | G, \mathcal{T}_2)} |\Delta\epsilon_T| \leq \sup_{P_2(\bm{x}) \perp G}  |\Delta\epsilon_T|,
\end{equation}
where \( \Delta\epsilon_T \) represents the difference in generalization error between \( M \) and \( \hat{M} \) on the target domain \(T\), expressed as \(\Delta\epsilon_T=\epsilon_T(M)-\epsilon_T(\hat{M})\). \( P_2(\bm{x}) \perp G \) denotes that \( P_2(\bm{x}) \) is independent of \( G \), meaning \( P_2(\bm{x}) = P(\bm{x}|{G',\mathcal{T}_2}) \) or \( P_2(\bm{x}) = P_{R}(\bm{x}) \), where \( G' \) is a text-to-image model different from \( G \) and \( P_{R}(\bm{x}) \) represents the distribution of real data. \( \mathcal{T}_1 \), \( \mathcal{T}_2 \) and \( \mathcal{T}_t \) are distinct sets of text prompts.
\label{th1}
\end{theorem}
Detailed proofs can be found in the supplementary material. \cref{th1} suggests that \textit{for two models with distinct (identical) training data sources, their generalization errors differ substantially (minimally)}. Its contrapositive states that \textit{for two models with small (large) discrepancies in generalization error, their training data sources are identical (distinct)}, which supports our method.
Specifically, when \( M \) and \( \hat{M} \) are trained on synthetic data from the same text-to-image model, their accuracies on the corresponding synthetic data from this text-to-image model are more akin, compared to when their training data is derived from disparate sources.
This aligns with the rationale of TrainProVe, as visually depicted in \cref{fig:motivation}. The discrepancies in accuracies among these models originate from the disparities in their source domains. The closer the data distributions of the source domain data, the more analogous the accuracies of the trained models on the target domain data.

\subsection{The Proposed TrainProVe}
\label{sec:pipeline}
Building upon the above analysis, we propose a method called TrainProVe to protect generative models’ intellectual property, as shown in \cref{fig:flowchart}, which consists of three stages:

(1) \textbf{Shadow and Validation Dataset Generation}: 
Utilizing $G_d$ to produce a shadow dataset $\mathcal{D}_{sdw}$ and a validation dataset $\mathcal{D}_v$, based on various text prompt sets $\mathcal{T}_{sdw}$ and $\mathcal{T}_v$, tailored to the predicted labels set $\mathcal{C}$ of the suspect model $M_{sus}$. Here $G_d$, $\mathcal{T}_{sdw}$, and $\mathcal{T}_v$ correspond to $G$, $\mathcal{T}_1$, and $\mathcal{T}_t$ in \cref{fig:motivation} respectively, while $\mathcal{D}_{sdw}$ and $\mathcal{D}_v$ correspond to the training data of $M$ and the target domain data.

(2) \textbf{ Shadow Model Training}: 
Using the shadow dataset $\mathcal{D}_{sdw}$ to train the shadow model $M_{sdw}$. The shadow model is analogous to $M$ in \cref{fig:motivation}, whereas the suspect model is denoted as $\hat{M}$. The purpose of the shadow model is to compare its performance with the suspect model trained on distinct data, focusing on their generalization capacities (accuracies) on the validation dataset. Specifically, when both models are trained on data from $G_d$, their accuracies on the validation dataset tend to be similar.

(3) \textbf{Hypothesis Testing}: 
Providing the validation dataset $\mathcal{D}_v$ to both the shadow model $M_{sdw}$ and the suspect model $M_{sus}$ in the same sequence batch by batch. Subsequently, calculating the accuracy of the two models batch by batch, leading to accuracy sets $\mathcal{A}_{sdw}$ and $\mathcal{A}_{sus}$. Ultimately, utilizing a one-sided Grubbs’ hypothesis test~\cite{grubbs1949sample} to ascertain whether $M_{sus}$ violates the intellectual property of $G_d$.


In the one-sided Grubbs' hypothesis test, the null hypothesis $H_0$ asserts that the mean of the set \( \mathcal{A}_{sus} \), denoted as $mean(\mathcal{A}_{sus})$, is not an unusually low value among the elements of $\mathcal{A}_{sdw}$. On the contrary, the alternative hypothesis $H_1$ posits that the mean of \( \mathcal{A}_{sus} \) is an exceptionally low value within $\mathcal{A}_{sdw}$. Rejecting the null hypothesis $H_0$ implies that the predictive capabilities of $M_{sus}$ are significantly inferior to those of $M_{sdw}$. This suggests that the model \( M_{sus} \) exhibits notably inferior generalization performance on the validation dataset compared to \( M_{sdw} \), indicating that $M_{sus}$ has not been exposed to the synthetic data from $G_d$ during training and is therefore compliant. Conversely, if the predictions generated by $M_{sus}$ and $M_{sdw}$ are similar, it indicates that the generalization performance of \( M_{sus} \) on the validation dataset closely resembles that of \( M_{sdw} \). This similarity implies that $M_{sus}$, much like $M_{sdw}$, has been trained on the synthetic data from $G_d$, thereby suggesting that $M_{sus}$ violates the intellectual property rights of $G_d$ and is considered illicit.

\begin{figure}[t]
  \centering
   \includegraphics[width=\linewidth, trim=1.15cm 0.7cm 1.95cm 0.7cm, clip]{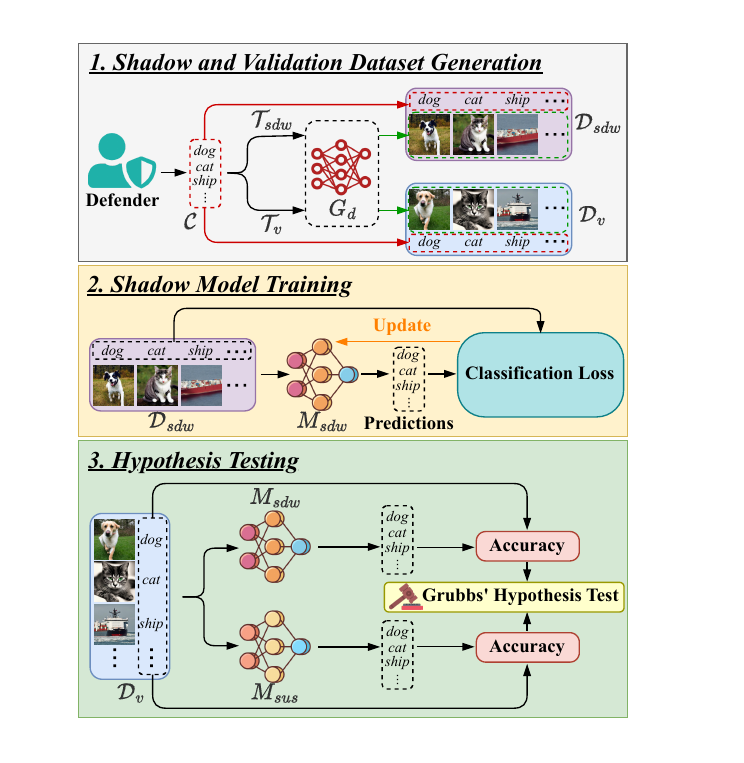}
   \vspace{-1.5em}
   \caption{The complete process of TrainProVe.}
   \vspace{-1.5em}
   \label{fig:flowchart}
\end{figure}

\section{Experiments}

Here, we present the results of DOV4MM on ImageNet~\cite{deng2009imagenet} and CIFAR~\cite{krizhevsky2009learning}, ablation experiments, interference resistance analysis, and efficiency analysis. More details and results (hypothesis testing, results on CLIP~\cite{radford2021learning}, etc) can be found in the supplementary materials.

\label{sec:exp}
\subsection{Experimental Setup}
\label{sec:para_setting}


\textbf{Models and Datasets}. We evaluate the proposed TrainProVe with three visual datasets (CIFAR10~\cite{krizhevsky2009learning}, CIFAR100~\cite{krizhevsky2009learning}, and ImageNet-100~\cite{sariyildiz2023fake}) and four text-to-image models (Stable Diffusion~\cite{rombach2022high} v1.4, latent consistency model~\cite{luo2023latent}, PixArt-$\alpha$~\cite{chenpixart}, and Stable Cascade~\cite{perniaswurstchen}). ImageNet-100 is a randomly selected subset of ImageNet-1K~\cite{deng2009imagenet}, encompassing 100 categories and 126,689 images. Stable Cascade is a latent diffusion model constructed on the Würstchen architecture~\cite{perniaswurstchen}. In the experiments, unless explicitly stated otherwise, all text-to-image models employed are pre-trained models. To distinguish from real datasets, we denote the synthetic datasets of CIFAR10, CIFAR100, and ImageNet-100 as CIFAR10-Syn, CIFAR100-Syn, and ImageNet-100-Syn, respectively. 


\noindent\textbf{Experiments on CIFAR}: The shadow dataset \( \mathcal{D}_{sdw} \) contains 30,000 images with a resolution of 32×32, and the text prompt is \( \mathcal{T}_{sdw}=\{``a \ \{\textit{c}\}"|\textit{c} \in \mathcal{C}\} \). The validation dataset \( \mathcal{D}_{v} \) includes 5,000 images with a resolution of 32×32, with a text prompt of \( \mathcal{T}_v=\{``a \ photo \ of \ a \ \{\textit{c}\}"|\textit{c} \in \mathcal{C}\} \). For the shadow model \( M_{sdw} \), the architecture is ResNet18~\cite{he2016deep}, and following previous work~\cite{shipard2023diversity}, we set the batch size for training \( M_{sdw} \) to 512, the learning rate to 1e-4, weight decay to 0.9, and the loss function to cross-entropy loss. 
To simulate the training data for the suspicious model, we generated 60,000 images with a resolution of 32×32 using the method in~\cite{shipard2023diversity} as CIFAR10-Syn/CIFAR100-Syn. Moreover, the suspicious models have various architectures and training hyperparameters to simulate different scenarios, specifically including: four architectures (ResNet50~\cite{he2016deep}, ResNet101~\cite{he2016deep}, ViT-S/16~\cite{dosovitskiy2020image}, ViT-B/16~\cite{dosovitskiy2020image}), two batch sizes (512 and 256), two learning rates (1e-4 and 5e-5), two weight decay values (0.9 and 0.09), and two loss functions (cross-entropy loss and focal loss~\cite{lin2017focal}). For each dataset, we can use these parameters to train 64 (\( C_4^1 C_2^1 C_2^1 C_2^1 C_2^1 \)) suspicious models. Batch size for inference in TrainProVe is 128, and training epoch for all models is set to 100. 
    
\noindent\textbf{Experiments on ImageNet-100}: For the shadow dataset \( \mathcal{D}_{sdw} \), validation dataset \( \mathcal{D}_{v} \), and shadow model \( M_{sdw} \), all settings are the same as those used for CIFAR, except that \( \mathcal{D}_{sdw} \) contains 50,000 images with a resolution of 256×256 and the batch size for model training is 64.
The training set for the suspicious model \( M_{sus} \) includes 100,000 images generated using the method in~\cite{sariyildiz2023fake}, with a resolution of 256×256. For the parameter selection of the suspicious model, all settings are the same as those used in the experiments on CIFAR, except that the batch size is 64 or 32. Therefore, we can also obtain 64 suspicious models for each dataset (ImageNet-100 or ImageNet-100-Syn).
Additionally, batch size for inference in TrainProVe is 16. The training epoch for all models is set to 100. 

\noindent\textbf{Evaluation Metrics.}
Our goal is to classify the suspicious model \( M_{sus} \) as either illegal or legal. Specifically, if the training data for \( M_{sus} \) originates from the defender's generative model \( G_d \), then \( M_{sus} \) should be classified as illegal; otherwise, \( M_{sus} \) is considered legal. Given that this is a classification task, we selected three commonly used metrics for evaluation: accuracy, F1 score, and AUROC.

\begin{table*}
\small
\centering
\addtolength{\tabcolsep}{-1.5pt}
\begin{tabular}{ccccccccc}

\Xhline{1.0pt}
 \multicolumn{2}{c}{\textbf{The Suspect's Data Sources}} &\multirow{2}{*}{\textbf{TP}} &\multirow{2}{*}{\textbf{FP}} &\multirow{2}{*}{\textbf{FN}} &\multirow{2}{*}{\textbf{TN}} &\multirow{2}{*}{\textbf{Accuracy}} &\multirow{2}{*}{\textbf{F1 Score}} &\multirow{2}{*}{\textbf{AUROC}} \\ \cline{1-2}
 \textbf{The Defender's Generative Model $G_d$} &\textbf{Data Sources Unrelated to \( G_d \)} & & & & & &  \\ \hline
 \multirow{4}{*}{Stable Diffusion~\cite{rombach2022high} v1.4} &Latent Consistency Model &\multirow{4}{*}{64} &\multirow{4}{*}{0} &0 &64 &\multirow{4}{*}{0.988} &\multirow{4}{*}{0.970} &\multirow{4}{*}{0.992} \\
 &PixArt-$\alpha$ & & &0 &64 & & & \\
 &Stable Cascade & & &0 &64 & & & \\
 &Real Data & & &4 &60 & & & \\ \cline{1-9}
 \multirow{4}{*}{Latent Consistency Model~\cite{luo2023latent}} &Stable Diffusion v1.4 &\multirow{4}{*}{64} &\multirow{4}{*}{0} &0 &64 &\multirow{4}{*}{1.000} &\multirow{4}{*}{1.000} &\multirow{4}{*}{1.000} \\
 &PixArt-$\alpha$ & & &0 &64 & & & \\
 &Stable Cascade & & &0 &64 & & & \\
 &Real Data & & &0 &64 & & & \\ \cline{1-9}
 \multirow{4}{*}{PixArt-$\alpha$~\cite{chenpixart}} &Stable Diffusion v1.4 &\multirow{4}{*}{64} &\multirow{4}{*}{0} &0 &64 &\multirow{4}{*}{1.000} &\multirow{4}{*}{1.000} &\multirow{4}{*}{1.000} \\
 &Latent Consistency Model & & &0 &64 & & & \\
 &Stable Cascade & & &0 &64 & & & \\
 &Real Data & & &0 &64 & & & \\ \cline{1-9}
 \multirow{4}{*}{Stable Cascade~\cite{perniaswurstchen}} &Stable Diffusion v1.4 &\multirow{4}{*}{64} &\multirow{4}{*}{0} &0 &64 &\multirow{4}{*}{1.000} &\multirow{4}{*}{1.000} &\multirow{4}{*}{1.000} \\
 &Latent Consistency Model & & &0 &64 & & & \\
 &PixArt-$\alpha$ & & &0 &64 & & & \\
 &Real Data & & &0 &64 & & & \\ \cline{1-9}
 \textbf{Average Value} & & & & & &0.997 &0.992 &0.998 \\
  \Xhline{1.0pt}
\end{tabular}
\vspace{-0.5em}
\caption{The results of TrainProVe on CIFAR10. The ``64'' in the table indicates the 64 suspicious models trained based on a data source, with different training hyperparameters. Suspicious models predicted as illegal/legal are considered positive/negative instances. If the suspicious model’s training data source is \( G_d \), it is illegal, otherwise, it is innocent.}
\vspace{-1.5em}
\label{tab:acc_cifar10}
\end{table*}

\subsection{Baselines}

Since TrainProVe is the first attempt at addressing the introduced training data provenance verification problem, there are no existing methods for direct comparison. Nonetheless, there are several baselines that can be utilized to showcase the efficacy of the proposed method, as enumerated below:

\noindent\textbf{Random classification}: a rather trivial baseline that randomly classify suspicious models as illegal or legal.
    
\noindent\textbf{Han \etal's work~\cite{han2024detection}}: This is the closest method to ours, which trains an attributor to classify suspicious models. The underlying assumption of this method is that the defender possesses multiple different generative models to train the attributor, while in our scenario, the defender has only a single generative model $G_d$. To evaluate Han \etal's method, we relaxed the constraints on the defender by allowing the defender to access all four text-to-image models.

\noindent\textbf{TrainProVe-Ent}: 
This method performs hypothesis testing by analyzing the entropy of the predicted logits generated by models $M_{sdw}$ and $M_{sus}$. Specifically, we analyze $\mathcal{D}_v$ in batches using models $M_{sdw}$ and $M_{sus}$ to derive their respective logits, and subsequently compute their entropies $\mathcal{E}_{sdw}$ and $\mathcal{E}_{sus}$. Following this, we employ a one-sided Grubbs' hypothesis test to ascertain whether the average of all elements in $\mathcal{E}_{sus}$ is an outlier with a high value compared to $\mathcal{E}_{sdw}$. This analysis aids in determining the validity of model $M_{sus}$, with its parameter configurations being consistent with those of TrainProVe.

\noindent\textbf{TrainProVe-Sim}: 
This method conducts hypothesis testing based on the pairwise similarity of logits within the same category from \( M_{sdw} \) and \( M_{sus} \). Specifically, \( \mathcal{D}_v \) is processed in batches using both \( M_{sdw} \) and \( M_{sus} \) to obtain their respective logits. Subsequently, the pairwise cosine similarity within the same category is computed to obtain the similarity sets \( \mathcal{S}_{sdw} \) and \( \mathcal{S}_{sus} \). Following this, a one-sided Grubbs' hypothesis test is conducted to determine whether the average of all elements in \( \mathcal{S}_{sus} \), is an outlier of low value in \( \mathcal{S}_{sdw} \), thereby evaluating the validity of \( M_{sus} \). The parameter configurations align with those in TrainProVe.

\subsection{Experimental Results}
In this paper, \( G_d \) is one of the four text-to-image models (Stable Diffusion v1.4, latent consistency model, PixArt-$\alpha$, and Stable Cascade), and the training dataset of \( M_{sus} \) comes from one of five data sources (Stable Diffusion v1.4, latent consistency model, PixArt-$\alpha$, Stable Cascade and real data). We first calculate the accuracy, F1 score, and AUROC under the four different \( G_d \), and then average the metrics across the four scenarios as the final metrics for comparing different methods. Taking the results of TrainProVe on CIFAR10 as an example (more experimental results can be found in the supplementary material), as shown in \cref{tab:acc_cifar10}. It is evident that TrainProVe can accurately classify suspicious models trained on various data sources.

The results of various methods on different datasets are  illustrated in \cref{tab:res}, showcasing the superiority of TrainProVe. In comparison to the research by Han \etal, TrainProVe has three advantages: (1) Han \etal's approach necessitates the defender to have access to multiple different generative models, while TrainProVe is effective even when the defender only has his own generative model \( G_d \); (2) The efficacy of Han \etal's method consistently lags behind that of TrainProVe across all cases. This is because the former relies on an attributor for classification, which tends to overfit to \(M_{sdw}\)'s architecture used by the defender. If the architecture of \(M_{sus}\) differs, this method is likely to fail. In contrast, TrainProVe relies on hypothesis testing to classify \(M_{sus}\), thus avoiding this risk; (3) Han \etal's method requires training a large number of models based on synthetic datasets from each generative model to train the attributor. In contrast, TrainProVe only requires training a single ResNet18 on one synthetic dataset generated by \( G_d \).

\begin{table}
\small
\centering
\addtolength{\tabcolsep}{-4pt}
\begin{tabular}{ccccc}

\toprule
 \textbf{Dataset} &\textbf{Method} &\textbf{Avg. Acc} &\textbf{Avg. F1} &\textbf{Avg. AUC} \\ \hline

\multirow{6}{*}{CIFAR10} &Random &0.500 &0.286 &0.500 \\
 &Han \etal's Work &0.768 &0.579 &0.731 \\
 &TrainProVe-Sim &0.812 &0.647 &0.821 \\
 &TrainProVe-Ent &0.868 &0.774 &0.902 \\
 &TrainProVe &\textbf{0.997} &\textbf{0.992} &\textbf{0.998} \\  \cline{1-5}
\multirow{6}{*}{CIFAR100} &Random &0.500 &0.286 &0.500 \\
 &Han \etal's Work &0.730 &0.548 &0.703 \\
 &TrainProVe-Sim &0.575 &0.318 &0.545 \\
 &TrainProVe-Ent &0.398 &0.408 &0.624 \\
 &TrainProVe &\textbf{0.992} &\textbf{0.979} &\textbf{0.981} \\  \cline{1-5}
 \multirow{6}{*}{ImageNet-100} &Random &0.500 &0.286 &0.500 \\
 &Han \etal's Work &0.671 &0.425 &0.645 \\
 &TrainProVe-Sim &0.555 &0.305 &0.528 \\
 &TrainProVe-Ent &0.356 &0.364 &0.567 \\
 &TrainProVe &\textbf{0.786} &\textbf{0.495} &\textbf{0.769} \\
  \bottomrule
\end{tabular}
\vspace{-1em}
\caption{The results of TrainProVe and baselines on different datasets. ``Avg. Acc'', ``Avg. F1'', and ``Avg. AUC'' are the average values of accuracy, F1 score, and AUROC under four different \( G_d \), similar to \cref{tab:acc_cifar10}. \textbf{Bold} represents the best results.}
\vspace{-0.5em}
\label{tab:res}
\end{table}

Moreover, compared to TrainProVe-Ent and TrainProVe-Sim, the efficacy of TrainProVe consistently surpasses that of the others. This is because it compares the predictions of $M_{sus}$ and $M_{sdw}$ with the true labels to obtain accuracy, using this as the basis for hypothesis testing. In contrast, TrainProVe-Ent and TrainProVe-Sim don't compare model predictions with true labels, resulting in a lack of information from the true labels, which leads to inferior validation performance compared to TrainProVe.

\subsection{Ablation Studies}

\paragraph{The model architecture of \(M_{sdw}\).} We substituted the standard architecture of \( M_{sdw} \), ResNet18, with ConvNeXt-B~\cite{liu2022convnet} and Swin-B~\cite{liu2021swin}, while leaving all other configurations unaltered. The outcomes are displayed in \cref{tab:Msdw_arch}, demonstrating that TrainProVe remains efficient irrespective of whether the architecture of the shadow model is based on CNNs or transformers.

\begin{table}
\small
\centering
\addtolength{\tabcolsep}{-4.5pt}
\begin{tabular}{ccccc}

\Xhline{1.0pt}
 \textbf{Dataset} &$M_{sdw}$ &\textbf{Avg. Acc} &\textbf{Avg. F1} &\textbf{Avg. AUC} \\ \hline

 \multirow{3}{*}{CIFAR10} &ResNet18~\cite{he2016deep} &0.997 &0.992 &0.998 \\
 &ConvNeXt-B~\cite{liu2022convnet} &0.995 &0.989 &0.997 \\
 &Swin-B~\cite{liu2021swin} &0.997 &0.992 &0.998 \\ \cline{1-5}
 \multirow{3}{*}{ImageNet-100} &ResNet18 &0.786 &0.495 &0.769 \\
 &ConvNeXt-B &0.771 &0.470 &0.756 \\
 &Swin-B &0.787 &0.487 &0.766 \\
  \Xhline{1.0pt}
\end{tabular}
\vspace{-1em}
\caption{Changing \(M_{sdw}\)'s architecture. ``Avg. Acc'', ``Avg. F1'', and ``Avg. AUC'' are the average values of accuracy, F1 score, and AUROC under four different \( G_d \), similar to \cref{tab:acc_cifar10}.}
\vspace{-2em}
\label{tab:Msdw_arch}
\end{table}

\paragraph{The training epochs of \(M_{sdw}\).} We explored TrainProVe's performance as the \( M_{sdw} \)'s training epochs (default is 100) increased from 0 to 200 in increments of 20. The findings on CIFAR10, depicted in \cref{fig:training_epochs_Msdw}, suggest that TrainProVe exhibits a commendable resilience to variations in the number of training epochs for \( M_{sdw} \). It is worth noting that an excessive or insufficient number of training epochs can lead to overfitting or underfitting of \( M_{sdw} \) on \( \mathcal{D}_{sdw} \), consequently resulting in a decline in TrainProVe's performance.
\begin{figure}[t]
  \centering
   \includegraphics[width=\linewidth, trim=15cm 5.3cm 9.5cm 2cm, clip]{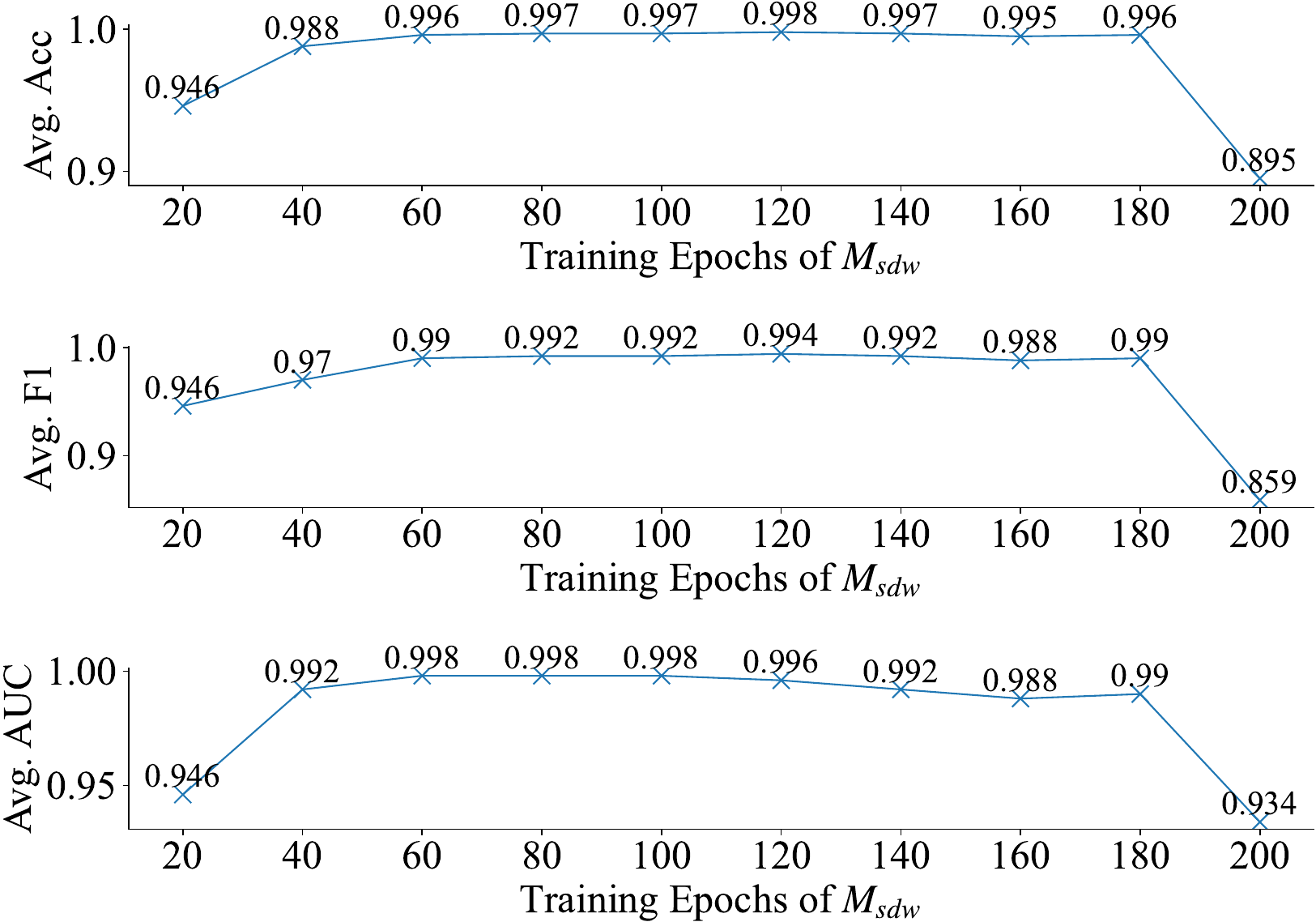}
   \vspace{-2em}
   \caption{Changing \( M_{sdw} \)'s training epochs. ``Avg. Acc'', ``Avg. F1'', and ``Avg. AUC'' are the average values of accuracy, F1 score, and AUROC under four different \( G_d \).}
   \vspace{-0.5em}
   \label{fig:training_epochs_Msdw}
\end{figure}

\paragraph{The text prompts of \(\mathcal{D}_{sdw}\) and \(\mathcal{D}_{v}\).} We swapped the text prompts, setting \( \mathcal{T}'_{sdw} = \mathcal{T}_{v} \) and \( \mathcal{T}'_{v} = \mathcal{T}_{sdw} \), and used \( \mathcal{T}'_{sdw} \) and \( \mathcal{T}'_{v} \) to generate \(\mathcal{D}_{sdw}\) and \(\mathcal{D}_{v}\), respectively. The results on CIFAR10, shown in \cref{tab:ablation_text}, indicate that TrainProVe is robust to the text prompts of \(\mathcal{D}_{sdw}\) and \(\mathcal{D}_{v}\).

\begin{table}
\centering
\small
\begin{tabular}{cccc}
\Xhline{1.0pt}
 \textbf{Text Prompts} &\textbf{Accuracy} &\textbf{F1 Score} &\textbf{AUROC} \\ \hline

 Original Prompts &0.988 &0.970 &0.992 \\
 Swapped Prompts &0.988 &0.970 &0.992 \\
  \Xhline{1.0pt}
\end{tabular}
\caption{Swapping $\mathcal{T}_{sdw}$ and $\mathcal{T}_{v}$. Note that Stable Diffusion v1.4 is \( G_d \). Latent consistency model, PixArt-$\alpha$, Stable Cascade, and real data represent other data sources.}
\label{tab:ablation_text}
\vspace{-2em}
\end{table}

\paragraph{The sample sizes of \(\mathcal{D}_{sdw}\) and \(\mathcal{D}_{v}\).} We generated \( \mathcal{D}_{sdw} \) and \( \mathcal{D}_{v} \) with different sample sizes for TrainProVe, keeping other settings unchanged. Note that the default sample sizes for \( \mathcal{D}_{sdw} \) and \( \mathcal{D}_{v} \) are 30,000 and 5,000 respectively, when generating CIFAR10. The results on CIFAR10 shown in \cref{fig:size_Dsdw_Dv} demonstrate that TrainProVe has good robustness to the sample sizes. Moreover, when \(\mathcal{D}_{sdw}\)'s sample size is too small, the generalization performance of \( M_{sdw} \) on \( \mathcal{D}_v \) may degrade, leading to a decline in TrainProVe's performance.
\begin{figure}[t]
  \centering
   \includegraphics[width=\linewidth, trim=27.4cm 0cm 4.5cm 2.85cm, clip]{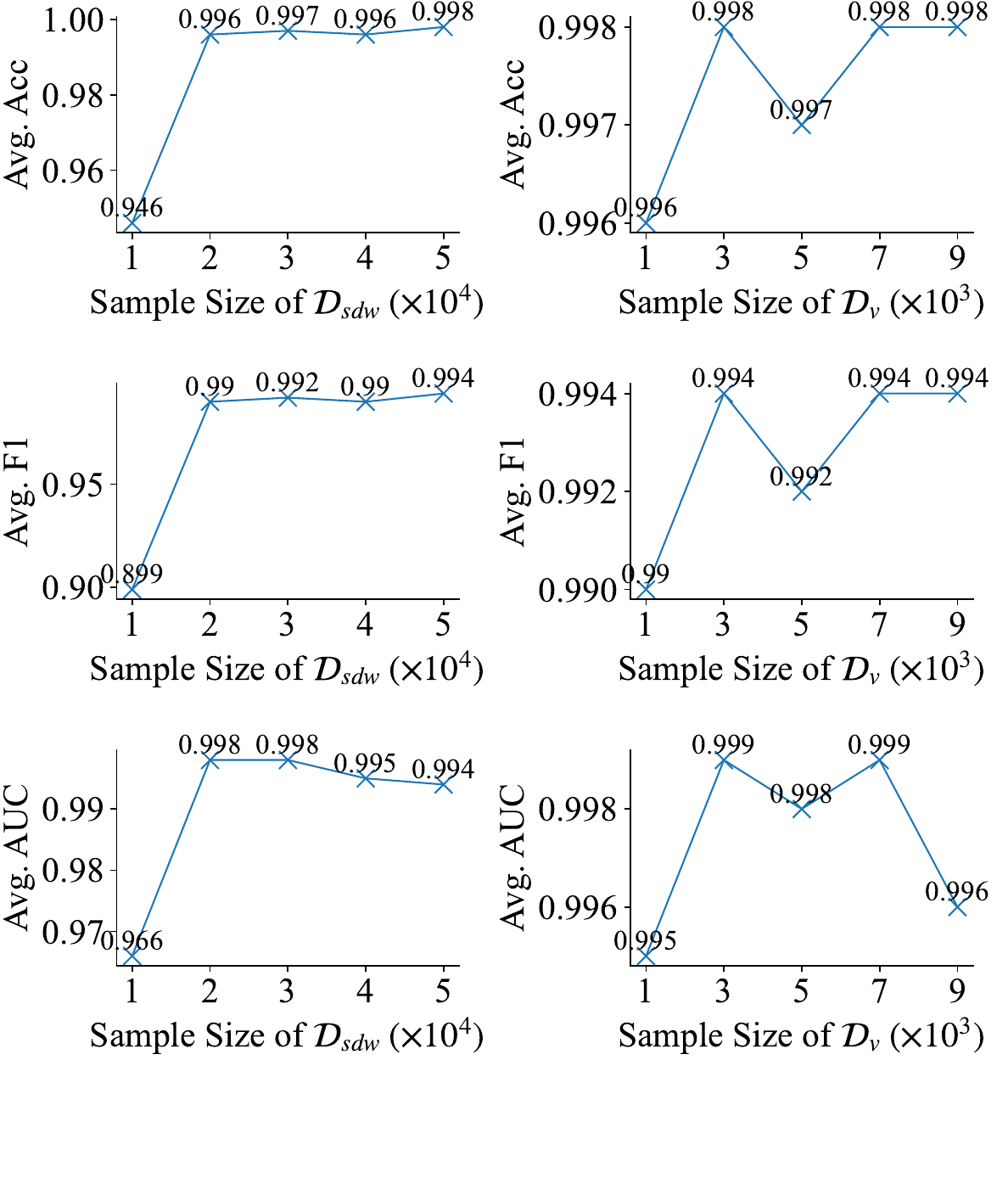}
   \vspace{-5em}
   \caption{Changing the sample sizes of \(\mathcal{D}_{sdw}\) and \(\mathcal{D}_{v}\). ``Avg. Acc'', ``Avg. F1'', and ``Avg. AUC'' are the average values of accuracy, F1 score, and AUROC under four different \( G_d \).}
   \vspace{-1.5em}
   \label{fig:size_Dsdw_Dv}
\end{figure}

\subsection{The Interference Resistance of TrainProVe}
In this section, we study that whether TrainProVe remains effective when facing more covert data theft. More experimental details can be found in the supplementary material.

\paragraph{Multiple training data sources for \(M_{sus}\).} We investigated the robustness of TrainProVe when the training dataset of \( M_{sus} \) has multiple data sources. Specifically, we assume that the training dataset of \( M_{sus} \) contains both synthetic data from \( G_d \) and real data. The results on CIFAR10 are shown in \cref{tab:multi_source}, where the training dataset for \( M_{sus} \) consists of 120,000 images, with synthetic data and real data each accounting for half. We calculated the accuracy of TrainProVe in classifying \( M_{sus} \) trained on this mixed dataset. The results indicate that TrainProVe can still accurately detect illegal behavior in this more difficult situation.

\paragraph{Fine-tuning \(G_d\).} When the suspect has access to a small amount of real data, he may fine-tune the defender’s open-source model \( G_d \) and then generate data to train \( M_{sus} \), which can enhance \( M_{sus} \)'s generalization ability on real data. Note that the defender only possesses the original \( G_d \) and has no knowledge of fine-tuning. Specifically, we fine-tuned the pre-trained Stable Diffusion v1.4 using 10,000 ImageNet-100 images (100 per class) and LoRA~\cite{hu2022lora} with 10 epochs, and then generated ImageNet-100-Syn to train \( M_{sus} \). The results are shown in \cref{tab:finetune_Gd}, which indicate that TrainProVe remains effective in this more arduous scenario.

\paragraph{Fine-tuning \(M_{sus}\).} Similarly, when the suspect has a small amount of real data, he may also fine-tune \( M_{sus} \) with real data, which was trained on synthetic data from \( G_d \), to enhance \( M_{sus} \)’s generalization ability on real data. Specifically, we fine-tuned \( M_{sus} \) using 10,000 ImageNet-100 images (100 per class), where \( M_{sus} \) was pre-trained on synthetic data from Stable Diffusion v1.4, with 10 fine-tuning epochs. The results, shown in \cref{tab:finetune_Msus}, indicate even in this more challenging scenario, TrainProVe remains effective.

\subsection{Efficiency Analysis}
We analyzed the efficiency of TrainProVe and the work of Han \etal on ImageNet-100. Specifically, following the setup in \cref{sec:para_setting}, given $M_{sus}$ trained on ImageNet-100 or ImageNet-100-Syn, we performed detection using both methods on an NVIDIA GeForce RTX 4090. $G_d$ is Stable Cascade. The runtime and results is shown in \cref{tab:efficiency}. For TrainProVe, \( \tau_g \) is the time for generating $\mathcal{D}_{sdw}$ and $\mathcal{D}_{v}$, \( \tau_t \) is the time for training $M_{sdw}$, and \( \tau_v \) is the time for classifying $M_{sus}$ using hypothesis testing. For Han \etal's method, \( \tau_g \) is the time for generating the training and validation dataset, \( \tau_t \) is the time for training the models and the attributor, and \( \tau_v \) is the time for classifying $M_{sus}$ using the attributor. From the results, it shows that the high efficiency of TrainProVe lies in its significantly smaller \( \tau_g \) and \( \tau_t \) compared to Han \etal's method. This is because TrainProVe only requires generating one dataset (\( \mathcal{D}_{sdw} \)) and training one model (\( M_{sdw} \)). In contrast, Han \etal's method requires generating \( m \) datasets and training \( m \times n \) models, where \( m \) is the number of text-to-image models accessible to the defender and \( n \) is the number of models trained by the defender on each dataset. These models are then used to train the attributor. Moreover, TrainProVe achieved better classification results, further illustrating its superiority.

\begin{table}
\small
\addtolength{\tabcolsep}{-4.7pt}
\begin{tabular}{ccccc}

\Xhline{1.0pt}
 \textbf{\( M_{sus} \)'s Training Source}  &$G_d$ &\textbf{TP} &\textbf{FP} &\textbf{Acc} \\ \hline

 \multirow{4}{*}{\makecell{CIFAR10-Syn\\+\\CIFAR10}}
 &Stable Diffusion v1.4 &64 &0 &1 \\
 &Latent Consistency Model &64 &0 &1 \\
 &PixArt-$\alpha$ &64 &0 &1 \\
 &Stable Cascade &64 &0 &1 \\
  \Xhline{1.0pt}
\end{tabular}
\vspace{-1em}
\caption{Results of TrainProVe when the training dataset of \( M_{sus} \) has multiple data sources. The ``64'' in the table indicates the 64 suspicious models with different training hyperparameters.}
\label{tab:multi_source}
\vspace{-0.5em}
\end{table}

\begin{table}
\centering
\small
\addtolength{\tabcolsep}{-3.5pt}
\begin{minipage}{0.23\textwidth}
\centering
\begin{tabular}{cccc}
\Xhline{1.0pt}
&\textbf{TP} &\textbf{FP} &\textbf{Acc} \\ \hline
 w/o fine-tuning &64 &0 &1.00 \\
 w/ fine-tuning &62 &2 &0.97 \\
  \Xhline{1.0pt}
\end{tabular}
\subcaption{Fine-tuning $G_d$}
\label{tab:finetune_Gd}
\end{minipage}
\hfill \hspace{2.5pt}
\begin{minipage}{0.23\textwidth}
\centering
\begin{tabular}{cccc}
\Xhline{1.0pt}
 &\textbf{TP} &\textbf{FP} &\textbf{Acc} \\ \hline
 w/o fine-tuning &64 &0 &1.00 \\
 w/ fine-tuning &57 &7 &0.89 \\
  \Xhline{1.0pt}
\end{tabular}
\subcaption{Fine-tuning $M_{sus}$}
\label{tab:finetune_Msus}
\end{minipage}
\vspace{-1em}
\caption{Fine-tuning \(G_d\) or \(M_{sus}\). The ``64'' in the table indicates the 64 suspicious models with different training hyperparameters.}
\vspace{-0.5em}
\end{table}

\begin{table}
\centering
\small
\addtolength{\tabcolsep}{-3.4pt}
\begin{tabular}{ccccc|ccc}
\Xhline{1.0pt}
 \textbf{Method} &$\tau_g$ &$\tau_t$ &$\tau_v$ &\textbf{Total} &\textbf{Acc} &\textbf{F1} &\textbf{AUC} \\ \hline

 Han \etal's Work &137.7h &8.0h &1.2h &146.9h &0.48 &0.28 &0.48 \\
 TrainProVe &33.6h &0.1h &1.4h &35.1h &0.78 &0.63 &0.84 \\
  \Xhline{1.0pt}
\end{tabular}
\vspace{-1em}
\caption{Efficiency analysis. ``Acc'', ``F1'' and ``AUC'' refer to the accuracy, F1 score and AUROC respectively.}
\vspace{-1.0em}
\label{tab:efficiency}
\end{table}

\section{Conclusion}
\label{sec:conclusion}

\begin{spacing}{1.00}
Open-source text-to-image models have reduced the cost of obtaining images. We introduce an important yet largely unexplored security issue: whether a suspicious model's training data originates from a specific text-to-image model? To our knowledge, we are the first to propose a solution to this problem, called TrainProVe. The success of TrainProVe in verifying training data provenance has been demonstrated through theory and experiments. Promising future work includes: (1) extending TrainProVe to other image generative models, such as image-to-image models; (2) adapting TrainProVe to protect other types of data (\eg, text); (3) exploring other privacy risks about text-to-image models.
\end{spacing}
\section{Acknowledgment}
This work was partially supported by the Pioneer R\&D Program of Zhejiang (No.2024C01021), and Zhejiang Province High-Level Talents Special Support Program ``Leading Talent of Technological Innovation of Ten-Thousands Talents Program" (No. 2022R52046).


{
    \small
    \bibliographystyle{ieeenat_fullname}
    \bibliography{main}

\begin{thebibliography}{56}
\providecommand{\natexlab}[1]{#1}
\providecommand{\url}[1]{\texttt{#1}}
\expandafter\ifx\csname urlstyle\endcsname\relax
  \providecommand{\doi}[1]{doi: #1}\else
  \providecommand{\doi}{doi: \begingroup \urlstyle{rm}\Url}\fi

\bibitem[Bansal and Grover(2023)]{bansal2023leaving}
Hritik Bansal and Aditya Grover.
\newblock Leaving reality to imagination: Robust classification via generated datasets.
\newblock \emph{arXiv preprint arXiv:2302.02503}, 2023.

\bibitem[Ben-David et~al.(2006)Ben-David, Blitzer, Crammer, and Pereira]{ben2006analysis}
Shai Ben-David, John Blitzer, Koby Crammer, and Fernando Pereira.
\newblock Analysis of representations for domain adaptation.
\newblock \emph{Advances in neural information processing systems}, 19, 2006.

\bibitem[Chen et~al.(2024{\natexlab{a}})Chen, Ge, Xie, Wu, Yao, Ren, Wang, Luo, Lu, and Li]{chen2024pixart}
Junsong Chen, Chongjian Ge, Enze Xie, Yue Wu, Lewei Yao, Xiaozhe Ren, Zhongdao Wang, Ping Luo, Huchuan Lu, and Zhenguo Li.
\newblock Pixart-$\sigma$: Weak-to-strong training of diffusion transformer for 4k text-to-image generation.
\newblock \emph{arXiv preprint arXiv:2403.04692}, 2024{\natexlab{a}}.

\bibitem[Chen et~al.(2024{\natexlab{b}})Chen, Jincheng, Chongjian, Yao, Xie, Wang, Kwok, Luo, Lu, and Li]{chenpixart}
Junsong Chen, YU Jincheng, GE Chongjian, Lewei Yao, Enze Xie, Zhongdao Wang, James Kwok, Ping Luo, Huchuan Lu, and Zhenguo Li.
\newblock Pixart-$\alpha$: Fast training of diffusion transformer for photorealistic text-to-image synthesis.
\newblock In \emph{The Twelfth International Conference on Learning Representations}, 2024{\natexlab{b}}.

\bibitem[Deng et~al.(2009)Deng, Dong, Socher, Li, Li, and Fei-Fei]{deng2009imagenet}
Jia Deng, Wei Dong, Richard Socher, Li-Jia Li, Kai Li, and Li Fei-Fei.
\newblock Imagenet: A large-scale hierarchical image database.
\newblock In \emph{2009 IEEE conference on computer vision and pattern recognition}, pages 248--255. Ieee, 2009.

\bibitem[Dockhorn et~al.(2023)Dockhorn, Rombach, Blatmann, and Yu]{dockhorn2023distilling}
Tim Dockhorn, Robin Rombach, Andreas Blatmann, and Yaoliang Yu.
\newblock Distilling the knowledge in diffusion models.
\newblock In \emph{CVPR Workshop Generative Modelsfor Computer Vision}, 2023.

\bibitem[Dosovitskiy et~al.(2020)Dosovitskiy, Beyer, Kolesnikov, Weissenborn, Zhai, Unterthiner, Dehghani, Minderer, Heigold, Gelly, et~al.]{dosovitskiy2020image}
Alexey Dosovitskiy, Lucas Beyer, Alexander Kolesnikov, Dirk Weissenborn, Xiaohua Zhai, Thomas Unterthiner, Mostafa Dehghani, Matthias Minderer, Georg Heigold, Sylvain Gelly, et~al.
\newblock An image is worth 16x16 words: Transformers for image recognition at scale.
\newblock \emph{arXiv preprint arXiv:2010.11929}, 2020.

\bibitem[Farahani et~al.(2021)Farahani, Voghoei, Rasheed, and Arabnia]{farahani2021brief}
Abolfazl Farahani, Sahar Voghoei, Khaled Rasheed, and Hamid~R Arabnia.
\newblock A brief review of domain adaptation.
\newblock \emph{Advances in data science and information engineering: proceedings from ICDATA 2020 and IKE 2020}, pages 877--894, 2021.

\bibitem[Grubbs(1949)]{grubbs1949sample}
Frank~Ephraim Grubbs.
\newblock \emph{Sample criteria for testing outlying observations}.
\newblock University of Michigan, 1949.

\bibitem[Han et~al.(2024)Han, Salem, Li, Guo, Backes, and Zhang]{han2024detection}
Ge Han, Ahmed Salem, Zheng Li, Shanqing Guo, Michael Backes, and Yang Zhang.
\newblock Detection and attribution of models trained on generated data.
\newblock In \emph{ICASSP 2024-2024 IEEE International Conference on Acoustics, Speech and Signal Processing (ICASSP)}, pages 4875--4879. IEEE, 2024.

\bibitem[He et~al.(2016)He, Zhang, Ren, and Sun]{he2016deep}
Kaiming He, Xiangyu Zhang, Shaoqing Ren, and Jian Sun.
\newblock Deep residual learning for image recognition.
\newblock In \emph{Proceedings of the IEEE conference on computer vision and pattern recognition}, pages 770--778, 2016.

\bibitem[He et~al.(2022)He, Sun, Yu, Xue, Zhang, Torr, Bai, and Qi]{he2022synthetic}
Ruifei He, Shuyang Sun, Xin Yu, Chuhui Xue, Wenqing Zhang, Philip Torr, Song Bai, and Xiaojuan Qi.
\newblock Is synthetic data from generative models ready for image recognition?
\newblock \emph{arXiv preprint arXiv:2210.07574}, 2022.

\bibitem[Hinton(2015)]{hinton2015distilling}
Geoffrey Hinton.
\newblock Distilling the knowledge in a neural network.
\newblock \emph{arXiv preprint arXiv:1503.02531}, 2015.

\bibitem[Ho et~al.(2020)Ho, Jain, and Abbeel]{ho2020denoising}
Jonathan Ho, Ajay Jain, and Pieter Abbeel.
\newblock Denoising diffusion probabilistic models.
\newblock \emph{Advances in neural information processing systems}, 33:\penalty0 6840--6851, 2020.

\bibitem[Hu et~al.(2022)Hu, Shen, Wallis, Allen-Zhu, Li, Wang, Wang, and Chen]{hu2022lora}
Edward~J Hu, Yelong Shen, Phillip Wallis, Zeyuan Allen-Zhu, Yuanzhi Li, Shean Wang, Lu Wang, and Weizhu Chen.
\newblock Lo{RA}: Low-rank adaptation of large language models.
\newblock In \emph{International Conference on Learning Representations}, 2022.

\bibitem[Jing et~al.(2021)Jing, Yang, Wang, Song, and Tao]{jing2021turning}
Yongcheng Jing, Yiding Yang, Xinchao Wang, Mingli Song, and Dacheng Tao.
\newblock Turning frequency to resolution: Video super-resolution via event cameras.
\newblock In \emph{Proceedings of the IEEE/CVF Conference on Computer Vision and Pattern Recognition}, pages 7772--7781, 2021.

\bibitem[Jing et~al.(2023)Jing, Yuan, Ju, Yang, Wang, and Tao]{jing2023deep}
Yongcheng Jing, Chongbin Yuan, Li Ju, Yiding Yang, Xinchao Wang, and Dacheng Tao.
\newblock Deep graph reprogramming.
\newblock In \emph{Proceedings of the IEEE/CVF Conference on Computer Vision and Pattern Recognition}, pages 24345--24354, 2023.

\bibitem[Kirillov et~al.(2023)Kirillov, Mintun, Ravi, Mao, Rolland, Gustafson, Xiao, Whitehead, Berg, Lo, et~al.]{kirillov2023segment}
Alexander Kirillov, Eric Mintun, Nikhila Ravi, Hanzi Mao, Chloe Rolland, Laura Gustafson, Tete Xiao, Spencer Whitehead, Alexander~C Berg, Wan-Yen Lo, et~al.
\newblock Segment anything.
\newblock In \emph{Proceedings of the IEEE/CVF International Conference on Computer Vision}, pages 4015--4026, 2023.

\bibitem[Krizhevsky(2009)]{krizhevsky2009learning}
A Krizhevsky.
\newblock Learning multiple layers of features from tiny images.
\newblock \emph{Master's thesis, University of Tront}, 2009.

\bibitem[Li et~al.(2023)Li, Li, and Hoi]{li2024blip}
Dongxu Li, Junnan Li, and Steven Hoi.
\newblock Blip-diffusion: Pre-trained subject representation for controllable text-to-image generation and editing.
\newblock \emph{Advances in Neural Information Processing Systems}, 36, 2023.

\bibitem[Lin(2017)]{lin2017focal}
T Lin.
\newblock Focal loss for dense object detection.
\newblock \emph{arXiv preprint arXiv:1708.02002}, 2017.

\bibitem[Liu et~al.(2024)Liu, Li, Wu, and Lee]{liu2024visual}
Haotian Liu, Chunyuan Li, Qingyang Wu, and Yong~Jae Lee.
\newblock Visual instruction tuning.
\newblock \emph{Advances in neural information processing systems}, 36, 2024.

\bibitem[Liu et~al.(2023)Liu, Li, Backes, Shen, and Zhang]{liu2023watermarking}
Yugeng Liu, Zheng Li, Michael Backes, Yun Shen, and Yang Zhang.
\newblock Watermarking diffusion model.
\newblock \emph{arXiv preprint arXiv:2305.12502}, 2023.

\bibitem[Liu et~al.(2021)Liu, Lin, Cao, Hu, Wei, Zhang, Lin, and Guo]{liu2021swin}
Ze Liu, Yutong Lin, Yue Cao, Han Hu, Yixuan Wei, Zheng Zhang, Stephen Lin, and Baining Guo.
\newblock Swin transformer: Hierarchical vision transformer using shifted windows.
\newblock In \emph{Proceedings of the IEEE/CVF international conference on computer vision}, pages 10012--10022, 2021.

\bibitem[Liu et~al.(2022)Liu, Mao, Wu, Feichtenhofer, Darrell, and Xie]{liu2022convnet}
Zhuang Liu, Hanzi Mao, Chao-Yuan Wu, Christoph Feichtenhofer, Trevor Darrell, and Saining Xie.
\newblock A convnet for the 2020s.
\newblock In \emph{Proceedings of the IEEE/CVF conference on computer vision and pattern recognition}, pages 11976--11986, 2022.

\bibitem[Luo et~al.(2023{\natexlab{a}})Luo, Huang, Zhang, Qian, Li, and Zhang]{luo2023steal}
Ge Luo, Junqiang Huang, Manman Zhang, Zhenxing Qian, Sheng Li, and Xinpeng Zhang.
\newblock Steal my artworks for fine-tuning? a watermarking framework for detecting art theft mimicry in text-to-image models.
\newblock \emph{arXiv preprint arXiv:2311.13619}, 2023{\natexlab{a}}.

\bibitem[Luo et~al.(2023{\natexlab{b}})Luo, Tan, Huang, Li, and Zhao]{luo2023latent}
Simian Luo, Yiqin Tan, Longbo Huang, Jian Li, and Hang Zhao.
\newblock Latent consistency models: Synthesizing high-resolution images with few-step inference.
\newblock \emph{arXiv preprint arXiv:2310.04378}, 2023{\natexlab{b}}.

\bibitem[Mekonnen et~al.(2024)Mekonnen, Dall'Asen, and Rota]{mekonnen2024adv}
Kidist~Amde Mekonnen, Nicola Dall'Asen, and Paolo Rota.
\newblock Adv-kd: Adversarial knowledge distillation for faster diffusion sampling.
\newblock \emph{arXiv preprint arXiv:2405.20675}, 2024.

\bibitem[Miller(1995)]{miller1995wordnet}
George~A Miller.
\newblock Wordnet: a lexical database for english.
\newblock \emph{Communications of the ACM}, 38\penalty0 (11):\penalty0 39--41, 1995.

\bibitem[Min et~al.(2024)Min, Li, Chen, and Cheng]{min2024watermark}
Rui Min, Sen Li, Hongyang Chen, and Minhao Cheng.
\newblock A watermark-conditioned diffusion model for ip protection.
\newblock \emph{arXiv preprint arXiv:2403.10893}, 2024.

\bibitem[Pan and Yang(2009)]{pan2009survey}
Sinno~Jialin Pan and Qiang Yang.
\newblock A survey on transfer learning.
\newblock \emph{IEEE Transactions on knowledge and data engineering}, 22\penalty0 (10):\penalty0 1345--1359, 2009.

\bibitem[Patil et~al.(2024)Patil, Berman, Rombach, and von Platen]{patil2024amused}
Suraj Patil, William Berman, Robin Rombach, and Patrick von Platen.
\newblock amused: An open muse reproduction.
\newblock \emph{arXiv preprint arXiv:2401.01808}, 2024.

\bibitem[Pernias et~al.(2024)Pernias, Rampas, Richter, Pal, and Aubreville]{perniaswurstchen}
Pablo Pernias, Dominic Rampas, Mats~L. Richter, Christopher~J. Pal, and Marc Aubreville.
\newblock Wuerstchen: An efficient architecture for large-scale text-to-image diffusion models.
\newblock In \emph{The Twelfth International Conference on Learning Representations}, 2024.

\bibitem[Radford et~al.(2021)Radford, Kim, Hallacy, Ramesh, Goh, Agarwal, Sastry, Askell, Mishkin, Clark, et~al.]{radford2021learning}
Alec Radford, Jong~Wook Kim, Chris Hallacy, Aditya Ramesh, Gabriel Goh, Sandhini Agarwal, Girish Sastry, Amanda Askell, Pamela Mishkin, Jack Clark, et~al.
\newblock Learning transferable visual models from natural language supervision.
\newblock In \emph{International conference on machine learning}, pages 8748--8763. PmLR, 2021.

\bibitem[Ramesh et~al.(2022)Ramesh, Dhariwal, Nichol, Chu, and Chen]{ramesh2022hierarchical}
Aditya Ramesh, Prafulla Dhariwal, Alex Nichol, Casey Chu, and Mark Chen.
\newblock Hierarchical text-conditional image generation with clip latents.
\newblock \emph{arXiv preprint arXiv:2204.06125}, 2022.

\bibitem[Ribeiro et~al.(2015)Ribeiro, Grolinger, and Capretz]{ribeiro2015mlaas}
Mauro Ribeiro, Katarina Grolinger, and Miriam~AM Capretz.
\newblock Mlaas: Machine learning as a service.
\newblock In \emph{2015 IEEE 14th international conference on machine learning and applications (ICMLA)}, pages 896--902. IEEE, 2015.

\bibitem[Rombach et~al.(2022)Rombach, Blattmann, Lorenz, Esser, and Ommer]{rombach2022high}
Robin Rombach, Andreas Blattmann, Dominik Lorenz, Patrick Esser, and Bjorn Ommer.
\newblock High-resolution image synthesis with latent diffusion models.
\newblock In \emph{Proceedings of the IEEE/CVF conference on computer vision and pattern recognition}, pages 10684--10695, 2022.

\bibitem[Salimans and Ho(2022)]{salimansprogressive}
Tim Salimans and Jonathan Ho.
\newblock Progressive distillation for fast sampling of diffusion models.
\newblock In \emph{International Conference on Learning Representations}, 2022.

\bibitem[Sar{\i}y{\i}ld{\i}z et~al.(2023)Sar{\i}y{\i}ld{\i}z, Alahari, Larlus, and Kalantidis]{sariyildiz2023fake}
Mert~B{\"u}lent Sar{\i}y{\i}ld{\i}z, Karteek Alahari, Diane Larlus, and Yannis Kalantidis.
\newblock Fake it till you make it: Learning transferable representations from synthetic imagenet clones.
\newblock In \emph{Proceedings of the IEEE/CVF Conference on Computer Vision and Pattern Recognition}, pages 8011--8021, 2023.

\bibitem[Schuhmann et~al.(2022)Schuhmann, Beaumont, Vencu, Gordon, Wightman, Cherti, Coombes, Katta, Mullis, Wortsman, et~al.]{schuhmann2022laion}
Christoph Schuhmann, Romain Beaumont, Richard Vencu, Cade Gordon, Ross Wightman, Mehdi Cherti, Theo Coombes, Aarush Katta, Clayton Mullis, Mitchell Wortsman, et~al.
\newblock Laion-5b: An open large-scale dataset for training next generation image-text models.
\newblock \emph{Advances in Neural Information Processing Systems}, 35:\penalty0 25278--25294, 2022.

\bibitem[Shipard et~al.(2023)Shipard, Wiliem, Thanh, Xiang, and Fookes]{shipard2023diversity}
Jordan Shipard, Arnold Wiliem, Kien~Nguyen Thanh, Wei Xiang, and Clinton Fookes.
\newblock Diversity is definitely needed: Improving model-agnostic zero-shot classification via stable diffusion.
\newblock In \emph{Proceedings of the IEEE/CVF Conference on Computer Vision and Pattern Recognition}, pages 769--778, 2023.

\bibitem[Sun et~al.(2015)Sun, Cui, Yong, Shen, and Chen]{sun2015mlaas}
Geng Sun, Tingru Cui, Jianming Yong, Jun Shen, and Shiping Chen.
\newblock Mlaas: a cloud-based system for delivering adaptive micro learning in mobile mooc learning.
\newblock \emph{IEEE Transactions on Services Computing}, 11\penalty0 (2):\penalty0 292--305, 2015.

\bibitem[Tancik et~al.(2020)Tancik, Mildenhall, and Ng]{tancik2020stegastamp}
Matthew Tancik, Ben Mildenhall, and Ren Ng.
\newblock Stegastamp: Invisible hyperlinks in physical photographs.
\newblock In \emph{Proceedings of the IEEE/CVF conference on computer vision and pattern recognition}, pages 2117--2126, 2020.

\bibitem[Torrey and Shavlik(2010)]{torrey2010transfer}
Lisa Torrey and Jude Shavlik.
\newblock Transfer learning.
\newblock In \emph{Handbook of research on machine learning applications and trends: algorithms, methods, and techniques}, pages 242--264. IGI global, 2010.

\bibitem[Wang and Deng(2018)]{wang2018deep}
Mei Wang and Weihong Deng.
\newblock Deep visual domain adaptation: A survey.
\newblock \emph{Neurocomputing}, 312:\penalty0 135--153, 2018.

\bibitem[Wang et~al.(2023{\natexlab{a}})Wang, Chen, Liu, Lyu, Metaxas, and Ma]{wang2023detect}
Zhenting Wang, Chen Chen, Yuchen Liu, Lingjuan Lyu, Dimitris Metaxas, and Shiqing Ma.
\newblock How to detect unauthorized data usages in text-to-image diffusion models.
\newblock \emph{arXiv preprint arXiv:2307.03108}, 2023{\natexlab{a}}.

\bibitem[Wang et~al.(2023{\natexlab{b}})Wang, Chen, Zeng, Lyu, and Ma]{wang2024did}
Zhenting Wang, Chen Chen, Yi Zeng, Lingjuan Lyu, and Shiqing Ma.
\newblock Where did i come from? origin attribution of ai-generated images.
\newblock \emph{Advances in neural information processing systems}, 36, 2023{\natexlab{b}}.

\bibitem[Weng et~al.(2022)Weng, Xiao, Yu, Wang, Wang, He, Li, Zhang, Lin, and Ding]{weng2022mlaas}
Qizhen Weng, Wencong Xiao, Yinghao Yu, Wei Wang, Cheng Wang, Jian He, Yong Li, Liping Zhang, Wei Lin, and Yu Ding.
\newblock Mlaas in the wild: Workload analysis and scheduling in large-scale heterogeneous gpu clusters.
\newblock In \emph{19th USENIX Symposium on Networked Systems Design and Implementation (NSDI 22)}, pages 945--960, 2022.

\bibitem[Xiang et~al.(2024)Xiang, Zhang, Shang, Wu, Yan, and Nie]{xiang2024dkdm}
Qianlong Xiang, Miao Zhang, Yuzhang Shang, Jianlong Wu, Yan Yan, and Liqiang Nie.
\newblock Dkdm: Data-free knowledge distillation for diffusion models with any architecture.
\newblock \emph{arXiv preprint arXiv:2409.03550}, 2024.

\bibitem[Xiong et~al.(2023)Xiong, Qin, Feng, and Zhang]{xiong2023flexible}
Cheng Xiong, Chuan Qin, Guorui Feng, and Xinpeng Zhang.
\newblock Flexible and secure watermarking for latent diffusion model.
\newblock In \emph{Proceedings of the 31st ACM International Conference on Multimedia}, pages 1668--1676, 2023.

\bibitem[Yuan et~al.(2022)Yuan, Pinto, Davies, and Torr]{yuan2022not}
Jianhao Yuan, Francesco Pinto, Adam Davies, and Philip Torr.
\newblock Not just pretty pictures: Toward interventional data augmentation using text-to-image generators.
\newblock \emph{arXiv preprint arXiv:2212.11237}, 2022.

\bibitem[Yuan et~al.(2024)Yuan, Li, Wang, and Zhang]{yuan2024watermarking}
Zihan Yuan, Li Li, Zichi Wang, and Xinpeng Zhang.
\newblock Watermarking for stable diffusion models.
\newblock \emph{IEEE Internet of Things Journal}, 2024.

\bibitem[Zhang et~al.(2024)Zhang, Wu, Zhang, Xu, Cao, Li, and Niu]{zhang2024training}
Likun Zhang, Hao Wu, Lingcui Zhang, Fengyuan Xu, Jin Cao, Fenghua Li, and Ben Niu.
\newblock Training data attribution: Was your model secretly trained on data created by mine?
\newblock \emph{arXiv preprint arXiv:2409.15781}, 2024.

\bibitem[Zhao et~al.(2023)Zhao, Pang, Du, Yang, Cheung, and Lin]{zhao2023recipe}
Yunqing Zhao, Tianyu Pang, Chao Du, Xiao Yang, Ngai-Man Cheung, and Min Lin.
\newblock A recipe for watermarking diffusion models.
\newblock \emph{arXiv preprint arXiv:2303.10137}, 2023.

\bibitem[Zhou et~al.(2023)Zhou, Sahak, and Ba]{zhou2023training}
Yongchao Zhou, Hshmat Sahak, and Jimmy Ba.
\newblock Training on thin air: Improve image classification with generated data.
\newblock \emph{arXiv preprint arXiv:2305.15316}, 2023.

\bibitem[Zhuang et~al.(2020)Zhuang, Qi, Duan, Xi, Zhu, Zhu, Xiong, and He]{zhuang2020comprehensive}
Fuzhen Zhuang, Zhiyuan Qi, Keyu Duan, Dongbo Xi, Yongchun Zhu, Hengshu Zhu, Hui Xiong, and Qing He.
\newblock A comprehensive survey on transfer learning.
\newblock \emph{Proceedings of the IEEE}, 109\penalty0 (1):\penalty0 43--76, 2020.

\end{thebibliography}
}

\clearpage
\setcounter{page}{1}
\maketitlesupplementary
\appendix
\setcounter{table}{0}
\renewcommand{\thetable}{A\arabic{table}}

\section{Proofs}
\label{sec:app_proofs}

\setcounter{theorem}{0}
\begin{theorem}
Assume \( M \) is trained on synthetic data generated by a text-to-image model \( G \), and a set of text prompts \( \mathcal{T}_1 \), i.e., \( P_1(\bm{x}) = P(\bm{x} | G, \mathcal{T}_1) \). \( \hat{M} \) can be trained on either real data or synthetic data generated by any text-to-image model. Based on the generalization errors of \( M \) and \( \hat{M} \) on the target domain \( T=\{ \mathcal{X}, P(\bm{x}| G, \mathcal{T}_t) \} \), we have:
\begin{equation}
    \sup_{P_2(\bm{x}) = P(\bm{x} | G, \mathcal{T}_2)} |\Delta\epsilon_T| \leq \sup_{P_2(\bm{x}) \perp G}  |\Delta\epsilon_T|,
\end{equation}
where \( \Delta\epsilon_T \) represents the difference in generalization error between \( M \) and \( \hat{M} \) on the target domain \(T\), expressed as \(\Delta\epsilon_T=\epsilon_T(M)-\epsilon_T(\hat{M})\). \( P_2(\bm{x}) \perp G \) denotes that \( P_2(\bm{x}) \) is independent of \( G \), meaning \( P_2(\bm{x}) = P(\bm{x}|{G',\mathcal{T}_2}) \) or \( P_2(\bm{x}) = P_{R}(\bm{x}) \), where \( G' \) is a text-to-image model different from \( G \) and \( P_{R}(\bm{x}) \) represents the distribution of real data. \( \mathcal{T}_1 \), \( \mathcal{T}_2 \) and \( \mathcal{T}_t \) are distinct sets of text prompts.
\end{theorem}

\begin{proof}
Theorem 1 in~\cite{ben2006analysis} states the upper bound of the generalization error of model \( M' \) on the target domain $T$. Let \( \mathcal{H} \) be a hypothesis space with VC dimension \( d' \), and let \( m' \) be the sample size of the dataset in the source domain $S$. Then, with probability at least \( 1 - \delta \), for every \( M' \in \mathcal{H} \):
\begin{equation}
    \epsilon_T(M') \leq \hat{\epsilon}_S(M') + \sqrt{\frac{4}{m'}\left(d'\log\frac{2em'}{d'}\right)}+d_{\mathcal{H}}(P_S,P_T)+\lambda,
\label{eq3}
\end{equation}
where \( \epsilon_T(M') \) and \( \hat{\epsilon_S}(M') \) represent the generalization error of $M'$ on the target domain and the empirical error of $M'$ on the source domain, respectively. \( P_S \) and \( P_T \) denote the marginal probability distributions of the source and target domains, respectively. \( d_\mathcal{H}(P_S,P_T) \) is the \( \mathcal{H} \)-divergence between the source and target domains, which measures the similarity between the two distributions (source and target) within the hypothesis space \( \mathcal{H} \). A smaller \( d_\mathcal{H}(P_S(\bm{x}),P_T(\bm{x})) \) indicates that the distributions of the source and target domains are closer. \(\lambda\) is a constant and $e$ is the base of the natural logarithm.

Substitute \( P_T= P(\bm{x}| G, \mathcal{T}_t) \) into \cref{eq3}, for \( M \) with the source domain distribution \( P(\bm{x}|G, \mathcal{T}_1) \):
\begin{equation}
\begin{split}
\epsilon_T(M) \leq \hat{\epsilon}_S(M) + \sqrt{\frac{4}{m}\left(d\log\frac{2em}{d}\right)} \\
+d_{\mathcal{H}}(P(\bm{x}|G, \mathcal{T}_1),P(\bm{x}| G, \mathcal{T}_t))+\lambda
\end{split}
\label{eq4}
\end{equation}
For \( \hat{M} \) with the source domain distribution \( P(x|G, \mathcal{T}_2) \), similarly:
\begin{equation}
\begin{split}
\epsilon_T(\hat{M}_1) \leq \hat{\epsilon}_S(\hat{M}_1) + \sqrt{\frac{4}{\hat{m}_1}\left(\hat{d}\log\frac{2e\hat{m}_1}{\hat{d}}\right)} \\
+d_{\mathcal{H}}(P(\bm{x}|G, \mathcal{T}_2),P(\bm{x}| G, \mathcal{T}_t))+\lambda
\end{split}
\label{eq5}
\end{equation}
To distinguish from later cases, we denote \( \hat{M} \) trained in this situation as \( \hat{M}_1 \) and \( m \) as \( \hat{m}_1 \).
For \( \hat{M} \) with the source domain distribution \( P(x|G', \mathcal{T}_2) \), similarly:
\begin{equation}
\begin{split}
\epsilon_T(\hat{M}_2) \leq \hat{\epsilon}_S(\hat{M}_2) + \sqrt{\frac{4}{\hat{m}_2}\left(\hat{d}\log\frac{2e\hat{m}_2}{\hat{d}}\right)} \\
+d_{\mathcal{H}}(P(\bm{x}|G', \mathcal{T}_2),P(\bm{x}| G, \mathcal{T}_t))+\lambda
\end{split}
\label{eq6}
\end{equation}
We denote \( \hat{M} \) trained in this situation as \( \hat{M}_2 \) and \( m \) as \( \hat{m}_2 \).

We denote the upper bounds in \cref{eq4}, \cref{eq5}, and \cref{eq6} as \( \xi \), \( \hat{\xi}_1 \), and \( \hat{\xi}_2 \), respectively. Then we calculate the upper bound of $|\epsilon_T(M)-\epsilon_T(\hat{M}_1)|$:
\begin{equation}
\sup{|\epsilon_T(M)-\epsilon_T(\hat{M}_1)|} = \max\{\xi,\hat{\xi}_1\}
\end{equation}
For the upper bound of $|\epsilon_T(M)-\epsilon_T(\hat{M}_2)|$:
\begin{equation}
\sup{|\epsilon_T(M)-\epsilon_T(\hat{M}_2)|} = \max\{\xi,\hat{\xi}_2\}
\end{equation}

Assuming that \( \hat{M} \) is trained on the same amount of source domain data in both cases, i.e., \( \hat{m}_1 = \hat{m}_2 \). Additionally, assuming the empirical errors on the training set are also the same, i.e., \( \hat{\epsilon}_S(\hat{M}_1) = \hat{\epsilon}_S(\hat{M}_2) \). Therefore, the only difference between \( \hat{\xi}_1 \) and \( \hat{\xi}_2 \) lies in the \( d_\mathcal{H}(\cdot,\cdot) \) term. Since the generative model in the conditions for \( P(\bm{x}|G, \mathcal{T}_2) \) and \( P(\bm{x}|G, \mathcal{T}_t) \) is the same while the text conditions differ, whereas both differ in the conditions for \( P(\bm{x}|G', \mathcal{T}_2) \) and \( P(\bm{x}|G, \mathcal{T}_t) \), it follows that:
\begin{equation}
d_{\mathcal{H}}(P(\bm{x}|G, \mathcal{T}_2),P(\bm{x}| G, \mathcal{T}_t)) \leq d_{\mathcal{H}}(P(\bm{x}|G', \mathcal{T}_2),P(\bm{x}| G, \mathcal{T}_t))
\end{equation}
Therefore, \( \hat{\xi}_1 \leq \hat{\xi}_2 \), i.e., 
\begin{equation}
\sup{|\epsilon_T(M)-\epsilon_T(\hat{M}_1)|} \leq \sup{|\epsilon_T(M)-\epsilon_T(\hat{M}_2)|}
\end{equation}
Alternatively, it can be expressed as:
\begin{equation}
\sup_{P_2(\bm{x}) = P(\bm{x} | G, \mathcal{T}_2)}{|\Delta\epsilon_T|} \leq \sup_{P_2(\bm{x}) = P(\bm{x} | G', \mathcal{T}_2)}{|\Delta\epsilon_T|},
\label{eq11}
\end{equation}
where \(\Delta\epsilon_T=\epsilon_T(M)-\epsilon_T(\hat{M})\).

Similarly, when \( \hat{M} \) is trained on real data \( \bm{x} \sim P_R(\bm{x}) \), we have:
\begin{equation}
\sup_{P_2(\bm{x}) = P(\bm{x} | G, \mathcal{T}_2)}{|\Delta\epsilon_T|} \leq \sup_{P_2(\bm{x}) = P_R(\bm{x})}{|\Delta\epsilon_T|}
\label{eq12}
\end{equation}
By combining \cref{eq11} and \cref{eq12}, we obtain:
\begin{equation}
    \sup_{P_2(\bm{x}) = P(\bm{x} | G, \mathcal{T}_2)} |\Delta\epsilon_T| \leq \sup_{P_2(\bm{x}) \perp G}  |\Delta\epsilon_T|,
\end{equation}
\end{proof}

\section{The Details of Experiments}
\label{sec:app_exp}

\subsection{Datasets Used}
\label{app_dataset}

\paragraph{CIFAR10}~\cite{krizhevsky2009learning}: The CIFAR10 dataset consists of 32x32 colored images with 10 classes. There are 50000 training images and 10000 test images. 

\paragraph{CIFAR100}~\cite{krizhevsky2009learning}: The CIFAR100 dataset consists of 32x32 coloured images with 100 classes. There are 50000 training images and 10000 test images. 

\paragraph{ImageNet-100}~\cite{sariyildiz2023fake}: A randomly chosen subset of ImageNet-1K~\cite{deng2009imagenet}, which has larger sized coloured images with 100 classes. There are approximately 126,689 training images and 5000 test images. As is commonly done, we resize all images to be of size 224x224. The specific categories we use are listed in ImageNet-100.txt in the supplementary material.

\subsection{Pre-trained Text-to-image Models Used}
\label{app_G}

We use four pre-trained text-to-image models, as follows:

\paragraph{Stable Diffusion v1.4}~\cite{rombach2022high}: Stable Diffusion v1.4 is trained on 512x512 images from a subset of the LAION-5B~\cite{schuhmann2022laion} dataset. This model uses a frozen CLIP ViT-L/14 text encoder to condition the model on text prompts. Its pre-trained weights can be obtained from \url{https://huggingface.co/CompVis/stable-diffusion-v1-4}.

\paragraph{Latent Consistency Model}~\cite{luo2023latent}: Latent consistency model is trained on 768x768 images from a subset of the LAION-5B~\cite{schuhmann2022laion} dataset, named LAION-Aesthetics. Its pre-trained weights can be obtained from \url{https://huggingface.co/SimianLuo/LCM_Dreamshaper_v7}.

\paragraph{Pixart-$\alpha$}~\cite{chenpixart}: Pixart-$\alpha$ is trained on images from SAM~\cite{kirillov2023segment} dataset, whose text prompts are generated by LLaVA~\cite{liu2024visual}. Its pre-trained weights can be obtained from \url{https://huggingface.co/PixArt-alpha/PixArt-XL-2-512x512}.

\paragraph{Stable Cascade}~\cite{perniaswurstchen}: Stable Cascade is built upon the Würstchen architecture~\cite{perniaswurstchen}, and trained on images from a subset of the LAION-5B~\cite{schuhmann2022laion} dataset. Its pre-trained weights can be obtained from \url{https://huggingface.co/stabilityai/stable-cascade}.

\subsection{Parameter Setting Details}
\label{app_setting_detail}

\begin{itemize}[leftmargin=*]
    \item \textbf{Experiments conducted on CIFAR10/CIFAR100}: We use the methods in~\cite{shipard2023diversity} to generate 60,000 images as CIFAR10-Syn and CIFAR100-Syn for training the suspicious models. Specifically, we generate 20,000 images each using the techniques of ``Class Prompt'', ``Multi-Domain'', and ``Random Unconditional Guidance''. Additionally, the shadow dataset and validation dataset retain default parameters except for the text prompts and resolution. Focal loss with \(\gamma = 2\) and \(\alpha = 0.25\). All models are trained using the AdamW optimizer.
    
    \item \textbf{Experiments conducted on ImageNet-100}: We use the methods in~\cite{sariyildiz2023fake} to generate 100,000 images as ImageNet-100-Syn for training the suspicious models. Specifically, Specifically, we use ``$c, d_c$'' to generate these images, where $c$ is the class name, and $d_c$ refers to the definition of class $c$ provided by WordNet~\cite{miller1995wordnet}. Guidance scale is 2. Additionally, the shadow dataset and validation dataset retain default parameters except for the text prompts and resolution.  Focal loss with \(\gamma = 2\) and \(\alpha = 0.25\). All models are trained using the AdamW optimizer.
\end{itemize}

\subsection{The Details of Han \textit{\textbf{et al.}}'s Work}
\label{app_hans_detail}
Unlike our setup, the work by Han \etal assumes that the defender has access to multiple different generative models to train an attributor. To enable comparison with this method, we followed this setup when using it. Specifically, we assume that the defender, in addition to having $G_d$, also possesses the other three text-to-image models mentioned in the paper (whereas in our setup, the defender has only one text-to-image model, $G_d$). Using these generative models, we create four shadow datasets following the same approach, as well as one validation dataset generated using $G_d$ in the same manner. 

Based on each shadow dataset, we train 16 shadow models, all of which are ResNet18 with varied training hyperparameters. In total, we obtain 64 shadow models. Using these shadow models, we infer logits on the validation dataset to train an attributor. The attributor is a two-layer fully connected network with an input dimension matching the dimension of the logits and an output dimension of 2. It classifies logits of shadow models trained on synthetic data generated by $G_d$ as 1 and others as 0. The attributor is trained for 10 epochs.

\begin{table*}
\small
\centering
\addtolength{\tabcolsep}{-1.5pt}
\begin{tabular}{ccccccccc}

\Xhline{1.0pt}
 \multicolumn{2}{c}{\textbf{The Suspect's Data Sources}} &\multirow{2}{*}{\textbf{TP}} &\multirow{2}{*}{\textbf{FP}} &\multirow{2}{*}{\textbf{FN}} &\multirow{2}{*}{\textbf{TN}} &\multirow{2}{*}{\textbf{Accuracy}} &\multirow{2}{*}{\textbf{F1 Score}} &\multirow{2}{*}{\textbf{AUROC}} \\ \cline{1-2}
 \textbf{The Defender's Generative Model} $G_d$ &\textbf{Data Sources Unrelated to} \( G_d \) & & & \\ \hline
\multirow{4}{*}{Stable Diffusion~\cite{rombach2022high} v1.4} 
 &Latent Consistency Model &\multirow{4}{*}{32} &\multirow{4}{*}{32} &32 &32 &\multirow{4}{*}{0.500} &\multirow{4}{*}{0.286} &\multirow{4}{*}{0.500} \\
 &PixArt-$\alpha$ & & &32 &32 & & & \\
 &Stable Cascade & & &32 &32 & & & \\
 &Real Data & & &32 &32 & & & \\ \cline{1-9}
\multirow{4}{*}{Latent Consistency Model~\cite{luo2023latent}} 
 &Stable Diffusion v1.4 &\multirow{4}{*}{32} &\multirow{4}{*}{32} &32 &32 &\multirow{4}{*}{0.500} &\multirow{4}{*}{0.286} &\multirow{4}{*}{0.500} \\
 &PixArt-$\alpha$ & & &32 &32 & & & \\
 &Stable Cascade & & &32 &32 & & & \\
 &Real Data & & &32 &32 & & & \\ \cline{1-9}
 \multirow{4}{*}{PixArt-$\alpha$~\cite{chenpixart}} 
 &Stable Diffusion v1.4 &\multirow{4}{*}{32} &\multirow{4}{*}{32} &32 &32 &\multirow{4}{*}{0.500} &\multirow{4}{*}{0.286} &\multirow{4}{*}{0.500} \\
 &Latent Consistency Model & & &32 &32 & & & \\
 &Stable Cascade & & &32 &32 & & & \\
 &Real Data & & &32 &32 & & & \\ \cline{1-9}
 \multirow{4}{*}{Stable Cascade~\cite{perniaswurstchen}} 
 &Stable Diffusion v1.4 &\multirow{4}{*}{32} &\multirow{4}{*}{32} &32 &32 &\multirow{4}{*}{0.500} &\multirow{4}{*}{0.286} &\multirow{4}{*}{0.500} \\
 &Latent Consistency Model & & &32 &32 & & & \\
 &PixArt-$\alpha$ & & &32 &32 & & & \\
 &Real Data & & &32 &32 & & & \\ \cline{1-9}
 \textbf{Average Value} & & & & & &0.500 &0.286 &0.500 \\
  \Xhline{1.0pt}
\end{tabular}
\vspace{-1em}
\caption{The results of random classification on CIFAR10.}
\label{tab:rand_cifar10}
\end{table*}

\subsection{The Details of Interference Resistance Experiments}
\paragraph{Multiple training data sources for \(M_{sus}\).} In this experiment, the methods for generating the synthetic dataset and training the network follow the settings in \cref{sec:para_setting}, with the only difference being that the training dataset for the suspicious model consists of both the synthetic and real datasets.

\paragraph{Fine-tuning \(G_{d}\).} In this experiment, we fine-tune \( G_d \) using LoRA with a batch size of 16, a learning rate of 1e-4 and a resolution of 512. The methods for generating dataset and training network follow the settings in \cref{sec:para_setting}.

\paragraph{Fine-tuning \(M_{sus}\).} In this experiment, we fine-tune the suspicious model with a learning rate of 1e-5, a weight decay of 1e-4, a batch size of 64, and the cross-entropy loss function. The methods for generating the synthetic dataset and training the network follow the settings in \cref{sec:para_setting}.

\subsection{Hypothesis Testing}
We represent the generalization errors of $M_{sdw}$ and $M_{sus}$ by the sets of accuracies $\mathcal{A}_{sdw}$ and $\mathcal{A}_{sus}$ for each batch on the validation dataset. Then we use the Grubbs test to determine whether the mean of $\mathcal{A}_{sus}$ is a low-value outlier of $\mathcal{A}_{sdw}$. The pseudocode is shown in \cref{alg}, where $t_{1-\frac{\alpha}{n},n-2}$ is the critical value of the t-distribution with degrees of freedom $n-2$ and significance level $\alpha$. When \( G > G_0 \), we consider \( mean(\mathcal{A}_{sus}) \) to be a low-value outlier of \( \mathcal{A}_{sdw} \), meaning the generalization error gap between two models is significant, and $M_{sus}$ is deemed legitimate. When \( G \leq G_0 \), \( mean(\mathcal{A}_{sus}) \) falls within the distribution of \( \mathcal{A}_{sdw} \), indicating that the generalization errors of two models are close, and $M_{sus}$ is deemed illegal.

\begin{algorithm}[ht]
\begin{algorithmic}[1] 
\caption{Grubbs’ Hypothesis Test}
\State \textbf{Input}: $\mathcal{A}_{sdw}$, $\mathcal{A}_{sus}$, significance level $\alpha=0.05$
\State $m_{sdw} = mean(\mathcal{A}_{sdw}), s_{sdw} = std(\mathcal{A}_{sdw})$
\State $n = len(\mathcal{A}_{sdw}), m_{sus} = mean(\mathcal{A}_{sus})$
\State $G = \frac{m_{sdw}-m_{sus}}{s_{sdw}}, G_0 = \frac{(n-1)}{\sqrt{n}}\sqrt{\frac{{(t_{1-\frac{\alpha}{n},n-2})}^2}{n-2+{(t_{1-\frac{\alpha}{n},n-2})}^2}}$
\State \textbf{Output}: $G>G_0$
\label{alg}
\end{algorithmic}
\end{algorithm}

\section{Results on CLIP}
We fine-tuned CLIP~\cite{radford2021learning} on ImageNet-100 using hyperparameters in \cref{sec:para_setting}, with ResNet50 and ViT-B. Then We calculate accuracies of TrainProVe and Han \etal's work on Stable Diffusion v1.4. Finally, the accuracy of TrainProVe reached 0.75, while Han \etal's work only achieved 0.56, meaning that compared to the baseline, TrainProVe can still be applied to more complex scenarios.

\section{More Experimental Results}
\label{app_more_res}
Here, we present the specific results of different methods on various datasets (as shown in \cref{tab:acc_cifar10} of the paper), as shown in \cref{tab:rand_cifar10} - \cref{tab:acc_imagenet100}.

\begin{table*}
\small
\centering
\vspace{-3em}
\addtolength{\tabcolsep}{-1.5pt}
\begin{tabular}{ccccccccc}

\Xhline{1.0pt}
  \multicolumn{2}{c}{\textbf{The Suspect's Data Sources}} &\multirow{2}{*}{\textbf{TP}} &\multirow{2}{*}{\textbf{FP}} &\multirow{2}{*}{\textbf{FN}} &\multirow{2}{*}{\textbf{TN}} &\multirow{2}{*}{\textbf{Accuracy}} &\multirow{2}{*}{\textbf{F1 Score}} &\multirow{2}{*}{\textbf{AUROC}} \\ \cline{1-2}
 \textbf{The Defender's Generative Model} $G_d$ &\textbf{Data Sources Unrelated to} \( G_d \) & & & & & &  \\ \hline
\multirow{4}{*}{Stable Diffusion~\cite{rombach2022high} v1.4} 
 &Latent Consistency Model &\multirow{4}{*}{60} &\multirow{4}{*}{4} &0 &64 &\multirow{4}{*}{0.756} &\multirow{4}{*}{0.606} &\multirow{4}{*}{0.824} \\
 &PixArt-$\alpha$ & & &0 &64 & & & \\
 &Stable Cascade & & &14 &50 & & & \\
 &Real Data & & &60 &4 & & & \\ \cline{1-9}
\multirow{4}{*}{Latent Consistency Model~\cite{luo2023latent}} 
 &Stable Diffusion v1.4 &\multirow{4}{*}{64} &\multirow{4}{*}{0} &4 &60 &\multirow{4}{*}{0.969} &\multirow{4}{*}{0.928} &\multirow{4}{*}{0.980} \\
 &PixArt-$\alpha$ & & &0 &64 & & & \\
 &Stable Cascade & & &2 &62 & & & \\
 &Real Data & & &4 &60 & & & \\ \cline{1-9}
 \multirow{4}{*}{PixArt-$\alpha$~\cite{chenpixart}} 
 &Stable Diffusion v1.4 &\multirow{4}{*}{47} &\multirow{4}{*}{17} &0 &64 &\multirow{4}{*}{0.919} &\multirow{4}{*}{0.783} &\multirow{4}{*}{0.850} \\
 &Latent Consistency Model & & &0 &64 & & & \\
 &Stable Cascade & & &9 &55 & & & \\
 &Real Data & & &0 &64 & & & \\ \cline{1-9}
 \multirow{4}{*}{Stable Cascade~\cite{perniaswurstchen}} 
 &Stable Diffusion v1.4 &\multirow{4}{*}{0} &\multirow{4}{*}{64} &60 &4 &\multirow{4}{*}{0.428} &\multirow{4}{*}{0.000} &\multirow{4}{*}{0.268} \\
 &Latent Consistency Model & & &53 &11 & & & \\
 &PixArt-$\alpha$ & & &0 &64 & & & \\
 &Real Data & & &6 &58 & & & \\\cline{1-9}
 \textbf{Average Value} & & & & & &0.768 &0.579 &0.731 \\
  \Xhline{1.0pt}
\end{tabular}
\vspace{-1em}
\caption{The results of Han \etal's method on CIFAR10.}
\vspace{-1.5em}
\label{tab:hans_cifar10}
\end{table*}

\begin{table*}
\small
\centering
\addtolength{\tabcolsep}{-1.5pt}
\begin{tabular}{ccccccccc}

\Xhline{1.0pt}
   \multicolumn{2}{c}{\textbf{The Suspect's Data Sources}} &\multirow{2}{*}{\textbf{TP}} &\multirow{2}{*}{\textbf{FP}} &\multirow{2}{*}{\textbf{FN}} &\multirow{2}{*}{\textbf{TN}} &\multirow{2}{*}{\textbf{Accuracy}} &\multirow{2}{*}{\textbf{F1 Score}} &\multirow{2}{*}{\textbf{AUROC}} \\ \cline{1-2}
 \textbf{The Defender's Generative Model} $G_d$ &\textbf{Data Sources Unrelated to} \( G_d \) & & &  \\ \hline
\multirow{4}{*}{Stable Diffusion~\cite{rombach2022high} v1.4} 
 &Latent Consistency Model &\multirow{4}{*}{64} &\multirow{4}{*}{0} &4 &60 &\multirow{4}{*}{0.941} &\multirow{4}{*}{0.871} &\multirow{4}{*}{0.963} \\
 &PixArt-$\alpha$ & & &4 &60 & & & \\
 &Stable Cascade & & &7 &57 & & & \\
 &Real Data & & &4 &60 & & & \\ \cline{1-9}
\multirow{4}{*}{Latent Consistency Model~\cite{luo2023latent}} 
 &Stable Diffusion v1.4 &\multirow{4}{*}{22} &\multirow{4}{*}{42} &0 &64 &\multirow{4}{*}{0.869} &\multirow{4}{*}{0.512} &\multirow{4}{*}{0.672} \\
 &PixArt-$\alpha$ & & &0 &64 & & & \\
 &Stable Cascade & & &0 &64 & & & \\
 &Real Data & & &0 &64 & & & \\ \cline{1-9}
 \multirow{4}{*}{PixArt-$\alpha$~\cite{chenpixart}} 
 &Stable Diffusion v1.4 &\multirow{4}{*}{64} &\multirow{4}{*}{0} &25 &39 &\multirow{4}{*}{0.609} &\multirow{4}{*}{0.506} &\multirow{4}{*}{0.756} \\
 &Latent Consistency Model & & &32 &32 & & & \\
 &Stable Cascade & & &64 &0 & & & \\
 &Real Data & & &4 &60 & & & \\ \cline{1-9}
 \multirow{4}{*}{Stable Cascade~\cite{perniaswurstchen}} 
 &Stable Diffusion v1.4 &\multirow{4}{*}{64} &\multirow{4}{*}{0} &17 &47 &\multirow{4}{*}{0.828} &\multirow{4}{*}{0.699} &\multirow{4}{*}{0.893} \\
 &Latent Consistency Model & & &5 &59 & & & \\
 &PixArt-$\alpha$ & & &19 &45 & & & \\
 &Real Data & & &14 &50 & & & \\ \cline{1-9}
 \textbf{Average Value} & & & & & &0.812 &0.647 &0.821 \\
  \Xhline{1.0pt}
\end{tabular}
\vspace{-1em}
\caption{The results of TrainProVe-Sim on CIFAR10.}
\vspace{-1.5em}
\label{tab:sim_cifar10}
\end{table*}

\begin{table*}
\small
\centering
\addtolength{\tabcolsep}{-1.5pt}
\begin{tabular}{ccccccccc}

\Xhline{1.0pt}
   \multicolumn{2}{c}{\textbf{The Suspect's Data Sources}} &\multirow{2}{*}{\textbf{TP}} &\multirow{2}{*}{\textbf{FP}} &\multirow{2}{*}{\textbf{FN}} &\multirow{2}{*}{\textbf{TN}} &\multirow{2}{*}{\textbf{Accuracy}} &\multirow{2}{*}{\textbf{F1 Score}} &\multirow{2}{*}{\textbf{AUROC}} \\ \cline{1-2}
 \textbf{The Defender's Generative Model} $G_d$ &\textbf{Data Sources Unrelated to} \( G_d \) & & & & & &  \\ \hline
\multirow{4}{*}{Stable Diffusion~\cite{rombach2022high} v1.4} 
 &Latent Consistency Model &\multirow{4}{*}{57} &\multirow{4}{*}{7} &1 &63 &\multirow{4}{*}{0.928} &\multirow{4}{*}{0.832} &\multirow{4}{*}{0.914} \\
 &PixArt-$\alpha$ & & &0 &64 & & & \\
 &Stable Cascade & & &11 &53 & & & \\
 &Real Data & & &4 &60 & & & \\ \cline{1-9}
\multirow{4}{*}{Latent Consistency Model~\cite{luo2023latent}} 
 &Stable Diffusion v1.4 &\multirow{4}{*}{60} &\multirow{4}{*}{4} &0 &64 &\multirow{4}{*}{0.988} &\multirow{4}{*}{0.968} &\multirow{4}{*}{0.969} \\
 &PixArt-$\alpha$ & & &0 &64 & & & \\
 &Stable Cascade & & &0 &64 & & & \\
 &Real Data & & &0 &64 & & & \\ \cline{1-9}
 \multirow{4}{*}{PixArt-$\alpha$~\cite{chenpixart}} 
 &Stable Diffusion v1.4 &\multirow{4}{*}{64} &\multirow{4}{*}{0} &9 &55 &\multirow{4}{*}{0.725} &\multirow{4}{*}{0.593} &\multirow{4}{*}{0.828} \\
 &Latent Consistency Model & & &22 &42 & & & \\
 &Stable Cascade & & &54 &10 & & & \\
 &Real Data & & &3 &61 & & & \\ \cline{1-9}
 \multirow{4}{*}{Stable Cascade~\cite{perniaswurstchen}} 
 &Stable Diffusion v1.4 &\multirow{4}{*}{64} &\multirow{4}{*}{0} &9 &55 &\multirow{4}{*}{0.831} &\multirow{4}{*}{0.703} &\multirow{4}{*}{0.895} \\
 &Latent Consistency Model & & &16 &48 & & & \\
 &PixArt-$\alpha$ & & &23 &41 & & & \\
 &Real Data & & &6 &58 & & & \\ \cline{1-9}
 \textbf{Average Value} & & & & & &0.868 &0.774 &0.902 \\
  \Xhline{1.0pt}
\end{tabular}
\vspace{-1em}
\caption{The results of TrainProVe-Ent on CIFAR10.}
\vspace{-1.5em}
\label{tab:ent_cifar10}
\end{table*}

\begin{table*}
\small
\centering
\vspace{-3em}
\addtolength{\tabcolsep}{-1.5pt}
\begin{tabular}{ccccccccc}

\Xhline{1.0pt}
   \multicolumn{2}{c}{\textbf{The Suspect's Data Sources}} &\multirow{2}{*}{\textbf{TP}} &\multirow{2}{*}{\textbf{FP}} &\multirow{2}{*}{\textbf{FN}} &\multirow{2}{*}{\textbf{TN}} &\multirow{2}{*}{\textbf{Accuracy}} &\multirow{2}{*}{\textbf{F1 Score}} &\multirow{2}{*}{\textbf{AUROC}} \\ \cline{1-2}
 \textbf{The Defender's Generative Model} $G_d$ &\textbf{Data Sources Unrelated to} \( G_d \) & & & & & &  \\ \hline
\multirow{4}{*}{Stable Diffusion~\cite{rombach2022high} v1.4} 
 &Latent Consistency Model &\multirow{4}{*}{32} &\multirow{4}{*}{32} &32 &32 &\multirow{4}{*}{0.500} &\multirow{4}{*}{0.286} &\multirow{4}{*}{0.500} \\
 &PixArt-$\alpha$ & & &32 &32 & & & \\
 &Stable Cascade & & &32 &32 & & & \\
 &Real Data & & &32 &32 & & & \\ \cline{1-9}
\multirow{4}{*}{Latent Consistency Model~\cite{luo2023latent}} 
 &Stable Diffusion v1.4 &\multirow{4}{*}{32} &\multirow{4}{*}{32} &32 &32 &\multirow{4}{*}{0.500} &\multirow{4}{*}{0.286} &\multirow{4}{*}{0.500} \\
 &PixArt-$\alpha$ & & &32 &32 & & & \\
 &Stable Cascade & & &32 &32 & & & \\
 &Real Data & & &32 &32 & & & \\ \cline{1-9}
 \multirow{4}{*}{PixArt-$\alpha$~\cite{chenpixart}} 
 &Stable Diffusion v1.4 &\multirow{4}{*}{32} &\multirow{4}{*}{32} &32 &32 &\multirow{4}{*}{0.500} &\multirow{4}{*}{0.286} &\multirow{4}{*}{0.500} \\
 &Latent Consistency Model & & &32 &32 & & & \\
 &Stable Cascade & & &32 &32 & & & \\
 &Real Data & & &32 &32 & & & \\ \cline{1-9}
 \multirow{4}{*}{Stable Cascade~\cite{perniaswurstchen}} 
 &Stable Diffusion v1.4 &\multirow{4}{*}{32} &\multirow{4}{*}{32} &32 &32 &\multirow{4}{*}{0.500} &\multirow{4}{*}{0.286} &\multirow{4}{*}{0.500} \\
 &Latent Consistency Model & & &32 &32 & & & \\
 &PixArt-$\alpha$ & & &32 &32 & & & \\
 &Real Data & & &32 &32 & & & \\ \cline{1-9}
 \textbf{Average Value} & & & & & &0.500 &0.286 &0.500 \\
  \Xhline{1.0pt}
\end{tabular}
\vspace{-1em}
\caption{The results of random classification on CIFAR100.}
\vspace{-1.5em}
\label{tab:rand_cifar100}
\end{table*}

\begin{table*}
\small
\centering
\addtolength{\tabcolsep}{-1.5pt}
\begin{tabular}{ccccccccc}

\Xhline{1.0pt}
   \multicolumn{2}{c}{\textbf{The Suspect's Data Sources}} &\multirow{2}{*}{\textbf{TP}} &\multirow{2}{*}{\textbf{FP}} &\multirow{2}{*}{\textbf{FN}} &\multirow{2}{*}{\textbf{TN}} &\multirow{2}{*}{\textbf{Accuracy}} &\multirow{2}{*}{\textbf{F1 Score}} &\multirow{2}{*}{\textbf{AUROC}} \\ \cline{1-2}
 \textbf{The Defender's Generative Model} $G_d$ &\textbf{Data Sources Unrelated to} \( G_d \) & & & & & &  \\ \hline
\multirow{4}{*}{Stable Diffusion~\cite{rombach2022high} v1.4} 
 &Latent Consistency Model &\multirow{4}{*}{64} &\multirow{4}{*}{0} &0 &64 &\multirow{4}{*}{0.772} &\multirow{4}{*}{0.637} &\multirow{4}{*}{0.861} \\
 &PixArt-$\alpha$ & & &9 &55 & & & \\
 &Stable Cascade & & &32 &32 & & & \\
 &Real Data & & &32 &32 & & & \\ \cline{1-9}
\multirow{4}{*}{Latent Consistency Model~\cite{luo2023latent}} 
 &Stable Diffusion v1.4 &\multirow{4}{*}{64} &\multirow{4}{*}{0} &0 &64 &\multirow{4}{*}{1.000} &\multirow{4}{*}{1.000} &\multirow{4}{*}{1.000} \\
 &PixArt-$\alpha$ & & &0 &64 & & & \\
 &Stable Cascade & & &0 &64 & & & \\
 &Real Data & & &0 &64 & & & \\ \cline{1-9}
 \multirow{4}{*}{PixArt-$\alpha$~\cite{chenpixart}} 
 &Stable Diffusion v1.4 &\multirow{4}{*}{36} &\multirow{4}{*}{28} &0 &64 &\multirow{4}{*}{0.791} &\multirow{4}{*}{0.518} &\multirow{4}{*}{0.705} \\
 &Latent Consistency Model & & &0 &64 & & & \\
 &Stable Cascade & & &32 &32 & & & \\
 &Real Data & & &7 &57 & & & \\ \cline{1-9}
 \multirow{4}{*}{Stable Cascade~\cite{perniaswurstchen}} 
 &Stable Diffusion v1.4 &\multirow{4}{*}{4} &\multirow{4}{*}{60} &28 &36 &\multirow{4}{*}{0.356} &\multirow{4}{*}{0.037} &\multirow{4}{*}{0.246} \\
 &Latent Consistency Model & & &64 &0 & & & \\
 &PixArt-$\alpha$ & & &47 &17 & & & \\
 &Real Data & & &7 &57 & & & \\ \cline{1-9}
 \textbf{Average Value} & & & & & &0.730 &0.548 &0.703 \\
  \Xhline{1.0pt}
\end{tabular}
\vspace{-1em}
\caption{The results of Han \etal's Method on CIFAR100.}
\vspace{-1.5em}
\label{tab:hans_cifar100}
\end{table*}

\begin{table*}
\small
\centering
\addtolength{\tabcolsep}{-1.5pt}
\begin{tabular}{ccccccccc}

\Xhline{1.0pt}
   \multicolumn{2}{c}{\textbf{The Suspect's Data Sources}} &\multirow{2}{*}{\textbf{TP}} &\multirow{2}{*}{\textbf{FP}} &\multirow{2}{*}{\textbf{FN}} &\multirow{2}{*}{\textbf{TN}} &\multirow{2}{*}{\textbf{Accuracy}} &\multirow{2}{*}{\textbf{F1 Score}} &\multirow{2}{*}{\textbf{AUROC}} \\ \cline{1-2}
 \textbf{The Defender's Generative Model} $G_d$ &\textbf{Data Sources Unrelated to} \( G_d \) & & & & & &  \\ \hline
\multirow{4}{*}{Stable Diffusion~\cite{rombach2022high} v1.4} 
 &Latent Consistency Model &\multirow{4}{*}{28} &\multirow{4}{*}{36} &32 &32 &\multirow{4}{*}{0.509} &\multirow{4}{*}{0.263} &\multirow{4}{*}{0.498} \\
 &PixArt-$\alpha$ & & &29 &35 & & & \\
 &Stable Cascade & & &32 &32 & & & \\
 &Real Data & & &28 &36 & & & \\ \cline{1-9}
\multirow{4}{*}{Latent Consistency Model~\cite{luo2023latent}} 
 &Stable Diffusion v1.4 &\multirow{4}{*}{32} &\multirow{4}{*}{32} &12 &52 &\multirow{4}{*}{0.672} &\multirow{4}{*}{0.379} &\multirow{4}{*}{0.607} \\
 &PixArt-$\alpha$ & & &12 &52 & & & \\
 &Stable Cascade & & &22 &42 & & & \\
 &Real Data & & &10 &54 & & & \\ \cline{1-9}
 \multirow{4}{*}{PixArt-$\alpha$~\cite{chenpixart}} 
 &Stable Diffusion v1.4 &\multirow{4}{*}{32} &\multirow{4}{*}{32} &32 &32 &\multirow{4}{*}{0.500} &\multirow{4}{*}{0.286} &\multirow{4}{*}{0.500} \\
 &Latent Consistency Model & & &32 &32 & & & \\
 &Stable Cascade & & &32 &32 & & & \\
 &Real Data & & &32 &32 & & & \\ \cline{1-9}
 \multirow{4}{*}{Stable Cascade~\cite{perniaswurstchen}} 
 &Stable Diffusion v1.4 &\multirow{4}{*}{32} &\multirow{4}{*}{32} &16 &48 &\multirow{4}{*}{0.619} &\multirow{4}{*}{0.344} &\multirow{4}{*}{0.574} \\
 &Latent Consistency Model & & &32 &32 & & & \\
 &PixArt-$\alpha$ & & &26 &38 & & & \\
 &Real Data & & &16 &48 & & & \\ \cline{1-9}
 \textbf{Average Value} & & & & & &0.575 &0.318 &0.545 \\
  \Xhline{1.0pt}
\end{tabular}
\vspace{-1em}
\caption{The results of TrainProVe-Sim on CIFAR100.}
\vspace{-1.5em}
\label{tab:sim_cifar100}
\end{table*}

\begin{table*}
\small
\centering
\vspace{-3em}
\addtolength{\tabcolsep}{-1.5pt}
\begin{tabular}{ccccccccc}

\Xhline{1.0pt}
   \multicolumn{2}{c}{\textbf{The Suspect's Data Sources}} &\multirow{2}{*}{\textbf{TP}} &\multirow{2}{*}{\textbf{FP}} &\multirow{2}{*}{\textbf{FN}} &\multirow{2}{*}{\textbf{TN}} &\multirow{2}{*}{\textbf{Accuracy}} &\multirow{2}{*}{\textbf{F1 Score}} &\multirow{2}{*}{\textbf{AUROC}} \\ \cline{1-2}
 \textbf{The Defender's Generative Model} $G_d$ &\textbf{Data Sources Unrelated to} \( G_d \) & & & & & &  \\ \hline
\multirow{4}{*}{Stable Diffusion~\cite{rombach2022high} v1.4} 
 &Latent Consistency Model &\multirow{4}{*}{64} &\multirow{4}{*}{0} &64 &0 &\multirow{4}{*}{0.347} &\multirow{4}{*}{0.380} &\multirow{4}{*}{0.592} \\
 &PixArt-$\alpha$ & & &55 &9 & & & \\
 &Stable Cascade & & &58 &6 & & & \\
 &Real Data & & &32 &32 & & & \\ \cline{1-9}
\multirow{4}{*}{Latent Consistency Model~\cite{luo2023latent}} 
 &Stable Diffusion v1.4 &\multirow{4}{*}{64} &\multirow{4}{*}{0} &25 &39 &\multirow{4}{*}{0.628} &\multirow{4}{*}{0.518} &\multirow{4}{*}{0.768} \\
 &PixArt-$\alpha$ & & &32 &32 & & & \\
 &Stable Cascade & & &48 &16 & & & \\
 &Real Data & & &14 &50 & & & \\ \cline{1-9}
 \multirow{4}{*}{PixArt-$\alpha$~\cite{chenpixart}} 
 &Stable Diffusion v1.4 &\multirow{4}{*}{64} &\multirow{4}{*}{0} &60 &4 &\multirow{4}{*}{0.253} &\multirow{4}{*}{0.349} &\multirow{4}{*}{0.533} \\
 &Latent Consistency Model & & &64 &0 & & & \\
 &Stable Cascade & & &64 &0 & & & \\
 &Real Data & & &51 &13 & & & \\ \cline{1-9}
 \multirow{4}{*}{Stable Cascade~\cite{perniaswurstchen}} 
 &Stable Diffusion v1.4 &\multirow{4}{*}{64} &\multirow{4}{*}{0} &41 &23 &\multirow{4}{*}{0.363} &\multirow{4}{*}{0.386} &\multirow{4}{*}{0.602} \\
 &Latent Consistency Model & & &62 &2 & & & \\
 &PixArt-$\alpha$ & & &64 &0 & & & \\
 &Real Data & & &37 &27 & & & \\ \cline{1-9}
 \textbf{Average Value} & & & & & &0.398 &0.408 &0.624 \\
  \Xhline{1.0pt}
\end{tabular}
\vspace{-1em}
\caption{The results of TrainProVe-Ent on CIFAR100.}
\vspace{-1.5em}
\label{tab:ent_cifar100}
\end{table*}

\begin{table*}
\small
\centering
\addtolength{\tabcolsep}{-1.5pt}
\begin{tabular}{cccccccccccc}

\Xhline{1.0pt}
   \multicolumn{2}{c}{\textbf{The Suspect's Data Sources}} &\multirow{2}{*}{\textbf{TP}} &\multirow{2}{*}{\textbf{FP}} &\multirow{2}{*}{\textbf{FN}} &\multirow{2}{*}{\textbf{TN}} &\multirow{2}{*}{\textbf{Accuracy}} &\multirow{2}{*}{\textbf{F1 Score}} &\multirow{2}{*}{\textbf{AUROC}} \\ \cline{1-2}
 \textbf{The Defender's Generative Model} $G_d$ &\textbf{Data Sources Unrelated to} \( G_d \) & & & & & &  \\ \hline
\multirow{4}{*}{Stable Diffusion~\cite{rombach2022high} v1.4} 
 &Latent Consistency Model &\multirow{4}{*}{64} &\multirow{4}{*}{0} &0 &64 &\multirow{4}{*}{0.988} &\multirow{4}{*}{0.970} &\multirow{4}{*}{1.000} \\
 &PixArt-$\alpha$ & & &0 &64 & & & \\
 &Stable Cascade & & &0 &64 & & & \\
 &Real Data & & &0 &64 & & & \\ \cline{1-9}
\multirow{4}{*}{Latent Consistency Model~\cite{luo2023latent}} 
 &Stable Diffusion v1.4 &\multirow{4}{*}{64} &\multirow{4}{*}{0} &0 &64 &\multirow{4}{*}{1.000} &\multirow{4}{*}{1.000} &\multirow{4}{*}{1.000} \\
 &PixArt-$\alpha$ & & &0 &64 & & & \\
 &Stable Cascade & & &0 &64 & & & \\
 &Real Data & & &0 &64 & & & \\ \cline{1-9}
 \multirow{4}{*}{PixArt-$\alpha$~\cite{chenpixart}} 
 &Stable Diffusion v1.4 &\multirow{4}{*}{55} &\multirow{4}{*}{9} &0 &64 &\multirow{4}{*}{0.972} &\multirow{4}{*}{0.924} &\multirow{4}{*}{0.930} \\
 &Latent Consistency Model & & &0 &64 & & & \\
 &Stable Cascade & & &0 &64 & & & \\
 &Real Data & & &0 &64 & & & \\ \cline{1-9}
 \multirow{4}{*}{Stable Cascade~\cite{perniaswurstchen}} 
 &Stable Diffusion v1.4 &\multirow{4}{*}{63} &\multirow{4}{*}{1} &0 &64 &\multirow{4}{*}{0.997} &\multirow{4}{*}{0.992} &\multirow{4}{*}{0.992} \\
 &Latent Consistency Model & & &0 &64 & & & \\
 &PixArt-$\alpha$ & & &0 &64 & & & \\
 &Real Data & & &0 &64 & & & \\ \cline{1-9}
 \textbf{Average Value} & & & & & &0.992 &0.979 &0.981 \\
  \Xhline{1.0pt}
\end{tabular}
\vspace{-1em}
\caption{The results of TrainProVe on CIFAR100.}
\vspace{-1.5em}
\label{tab:acc_cifar100}
\end{table*}

\begin{table*}
\small
\centering
\addtolength{\tabcolsep}{-1.5pt}
\begin{tabular}{ccccccccc}

\Xhline{1.0pt}
   \multicolumn{2}{c}{\textbf{The Suspect's Data Sources}} &\multirow{2}{*}{\textbf{TP}} &\multirow{2}{*}{\textbf{FP}} &\multirow{2}{*}{\textbf{FN}} &\multirow{2}{*}{\textbf{TN}} &\multirow{2}{*}{\textbf{Accuracy}} &\multirow{2}{*}{\textbf{F1 Score}} &\multirow{2}{*}{\textbf{AUROC}} \\ \cline{1-2}
 \textbf{The Defender's Generative Model} $G_d$ &\textbf{Data Sources Unrelated to} \( G_d \) & & & & & &  \\ \hline
\multirow{4}{*}{Stable Diffusion~\cite{rombach2022high} v1.4} 
 &Latent Consistency Model &\multirow{4}{*}{32} &\multirow{4}{*}{32} &32 &32 &\multirow{4}{*}{0.500} &\multirow{4}{*}{0.286} &\multirow{4}{*}{0.500} \\
 &PixArt-$\alpha$ & & &32 &32 & & & \\
 &Stable Cascade & & &32 &32 & & & \\
 &Real Data & & &32 &32 & & & \\ \cline{1-9}
\multirow{4}{*}{Latent Consistency Model~\cite{luo2023latent}} 
 &Stable Diffusion v1.4 &\multirow{4}{*}{32} &\multirow{4}{*}{32} &32 &32 &\multirow{4}{*}{0.500} &\multirow{4}{*}{0.286} &\multirow{4}{*}{0.500} \\
 &PixArt-$\alpha$ & & &32 &32 & & & \\
 &Stable Cascade & & &32 &32 & & & \\
 &Real Data & & &32 &32 & & & \\ \cline{1-9}
 \multirow{4}{*}{PixArt-$\alpha$~\cite{chenpixart}} 
 &Stable Diffusion v1.4 &\multirow{4}{*}{32} &\multirow{4}{*}{32} &32 &32 &\multirow{4}{*}{0.500} &\multirow{4}{*}{0.286} &\multirow{4}{*}{0.500} \\
 &Latent Consistency Model & & &32 &32 & & & \\
 &Stable Cascade & & &32 &32 & & & \\
 &Real Data & & &32 &32 & & & \\ \cline{1-9}
 \multirow{4}{*}{Stable Cascade~\cite{perniaswurstchen}} 
 &Stable Diffusion v1.4 &\multirow{4}{*}{32} &\multirow{4}{*}{32} &32 &32 &\multirow{4}{*}{0.500} &\multirow{4}{*}{0.286} &\multirow{4}{*}{0.500} \\
 &Latent Consistency Model & & &32 &32 & & & \\
 &PixArt-$\alpha$ & & &32 &32 & & & \\
 &Real Data & & &32 &32 & & & \\ \cline{1-9}
 \textbf{Average Value} & & & & & &0.500 &0.286 &0.500 \\
  \Xhline{1.0pt}
\end{tabular}
\vspace{-1em}
\caption{The results of random classification on ImageNet-100.}
\vspace{-1.5em}
\label{tab:rand_imagenet100}
\end{table*}

\begin{table*}
\small
\centering
\vspace{-3em}
\addtolength{\tabcolsep}{-1.5pt}
\begin{tabular}{ccccccccc}

\Xhline{1.0pt}
   \multicolumn{2}{c}{\textbf{The Suspect's Data Sources}} &\multirow{2}{*}{\textbf{TP}} &\multirow{2}{*}{\textbf{FP}} &\multirow{2}{*}{\textbf{FN}} &\multirow{2}{*}{\textbf{TN}} &\multirow{2}{*}{\textbf{Accuracy}} &\multirow{2}{*}{\textbf{F1 Score}} &\multirow{2}{*}{\textbf{AUROC}} \\ \cline{1-2}
 \textbf{The Defender's Generative Model} $G_d$ &\textbf{Data Sources Unrelated to} \( G_d \) & & & & & &  \\ \hline
\multirow{4}{*}{Stable Diffusion~\cite{rombach2022high} v1.4} 
 &Latent Consistency Model &\multirow{4}{*}{32} &\multirow{4}{*}{32} &31 &33 &\multirow{4}{*}{0.684} &\multirow{4}{*}{0.388} &\multirow{4}{*}{0.615} \\
 &PixArt-$\alpha$ & & &0 &64 & & & \\
 &Stable Cascade & & &6 &58 & & & \\
 &Real Data & & &32 &32 & & & \\ \cline{1-9}
\multirow{4}{*}{Latent Consistency Model~\cite{luo2023latent}} 
 &Stable Diffusion v1.4 &\multirow{4}{*}{26} &\multirow{4}{*}{38} &15 &49 &\multirow{4}{*}{0.772} &\multirow{4}{*}{0.416} &\multirow{4}{*}{0.635} \\
 &PixArt-$\alpha$ & & &6 &58 & & & \\
 &Stable Cascade & & &1 &63 & & & \\
 &Real Data & & &13 &51 & & & \\ \cline{1-9}
 \multirow{4}{*}{PixArt-$\alpha$~\cite{chenpixart}} 
 &Stable Diffusion v1.4 &\multirow{4}{*}{64} &\multirow{4}{*}{0} &24 &40 &\multirow{4}{*}{0.753} &\multirow{4}{*}{0.618} &\multirow{4}{*}{0.846} \\
 &Latent Consistency Model & & &3 &61 & & & \\
 &Stable Cascade & & &32 &32 & & & \\
 &Real Data & & &20 &44 & & & \\ \cline{1-9}
 \multirow{4}{*}{Stable Cascade~\cite{perniaswurstchen}} 
 &Stable Diffusion v1.4 &\multirow{4}{*}{32} &\multirow{4}{*}{32} &32 &32 &\multirow{4}{*}{0.475} &\multirow{4}{*}{0.276} &\multirow{4}{*}{0.484} \\
 &Latent Consistency Model & & &40 &24 & & & \\
 &PixArt-$\alpha$ & & &32 &32 & & & \\
 &Real Data & & &32 &32 & & & \\ \cline{1-9}
 \textbf{Average Value} & & & & & &0.671 &0.425 &0.645 \\
  \Xhline{1.0pt}
\end{tabular}
\vspace{-1em}
\caption{The results of Han \etal's Method on ImageNet-100.}
\vspace{-1.5em}
\label{tab:hans_imagenet100}
\end{table*}

\begin{table*}
\small
\centering
\addtolength{\tabcolsep}{-1.5pt}
\begin{tabular}{cccccccccccc}

\Xhline{1.0pt}
   \multicolumn{2}{c}{\textbf{The Suspect's Data Sources}} &\multirow{2}{*}{\textbf{TP}} &\multirow{2}{*}{\textbf{FP}} &\multirow{2}{*}{\textbf{FN}} &\multirow{2}{*}{\textbf{TN}} &\multirow{2}{*}{\textbf{Accuracy}} &\multirow{2}{*}{\textbf{F1 Score}} &\multirow{2}{*}{\textbf{AUROC}} \\ \cline{1-2}
 \textbf{The Defender's Generative Model} $G_d$ &\textbf{Data Sources Unrelated to} \( G_d \) & & & & & &  \\ \hline
\multirow{4}{*}{Stable Diffusion~\cite{rombach2022high} v1.4} 
 &Latent Consistency Model &\multirow{4}{*}{32} &\multirow{4}{*}{32} &32 &32 &\multirow{4}{*}{0.506} &\multirow{4}{*}{0.288} &\multirow{4}{*}{0.504} \\
 &PixArt-$\alpha$ & & &30 &34 & & & \\
 &Stable Cascade & & &32 &32 & & & \\
 &Real Data & & &32 &32 & & & \\ \cline{1-9}
\multirow{4}{*}{Latent Consistency Model~\cite{luo2023latent}} 
 &Stable Diffusion v1.4 &\multirow{4}{*}{28} &\multirow{4}{*}{36} &16 &48 &\multirow{4}{*}{0.641} &\multirow{4}{*}{0.327} &\multirow{4}{*}{0.564} \\
 &PixArt-$\alpha$ & & &16 &48 & & & \\
 &Stable Cascade & & &26 &38 & & & \\
 &Real Data & & &21 &43 & & & \\ \cline{1-9}
 \multirow{4}{*}{PixArt-$\alpha$~\cite{chenpixart}} 
 &Stable Diffusion v1.4 &\multirow{4}{*}{32} &\multirow{4}{*}{32} &28 &36 &\multirow{4}{*}{0.519} &\multirow{4}{*}{0.294} &\multirow{4}{*}{0.512} \\
 &Latent Consistency Model & & &32 &32 & & & \\
 &Stable Cascade & & &32 &32 & & & \\
 &Real Data & & &30 &34 & & & \\ \cline{1-9}
 \multirow{4}{*}{Stable Cascade~\cite{perniaswurstchen}} 
 &Stable Diffusion v1.4 &\multirow{4}{*}{32} &\multirow{4}{*}{32} &26 &38 &\multirow{4}{*}{0.553} &\multirow{4}{*}{0.309} &\multirow{4}{*}{0.533} \\
 &Latent Consistency Model & & &31 &33 & & & \\
 &PixArt-$\alpha$ & & &28 &36 & & & \\
 &Real Data & & &26 &38 & & & \\ \cline{1-9}
 \textbf{Average Value} & & & & & &0.555 &0.305 &0.528 \\
  \Xhline{1.0pt}
\end{tabular}
\vspace{-1em}
\caption{The results of TrainProVe-Sim on ImageNet-100.}
\vspace{-1.5em}
\label{tab:sim_imagenet100}
\end{table*}

\begin{table*}
\small
\centering
\addtolength{\tabcolsep}{-1.5pt}
\begin{tabular}{ccccccccc}

\Xhline{1.0pt}
   \multicolumn{2}{c}{\textbf{The Suspect's Data Sources}} &\multirow{2}{*}{\textbf{TP}} &\multirow{2}{*}{\textbf{FP}} &\multirow{2}{*}{\textbf{FN}} &\multirow{2}{*}{\textbf{TN}} &\multirow{2}{*}{\textbf{Accuracy}} &\multirow{2}{*}{\textbf{F1 Score}} &\multirow{2}{*}{\textbf{AUROC}} \\ \cline{1-2}
 \textbf{The Defender's Generative Model} $G_d$ &\textbf{Data Sources Unrelated to} \( G_d \) & & & & & &  \\ \hline
\multirow{4}{*}{Stable Diffusion~\cite{rombach2022high} v1.4} 
 &Latent Consistency Model &\multirow{4}{*}{64} &\multirow{4}{*}{0} &53 &11 &\multirow{4}{*}{0.266} &\multirow{4}{*}{0.353} &\multirow{4}{*}{0.541} \\
 &PixArt-$\alpha$ & & &60 &4 & & & \\
 &Stable Cascade & & &62 &2 & & & \\
 &Real Data & & &60 &4 & & & \\ \cline{1-9}
\multirow{4}{*}{Latent Consistency Model~\cite{luo2023latent}} 
 &Stable Diffusion v1.4 &\multirow{4}{*}{51} &\multirow{4}{*}{13} &41 &23 &\multirow{4}{*}{0.441} &\multirow{4}{*}{0.363} &\multirow{4}{*}{0.574} \\
 &PixArt-$\alpha$ & & &48 &16 & & & \\
 &Stable Cascade & & &37 &27 & & & \\
 &Real Data & & &40 &24 & & & \\ \cline{1-9}
 \multirow{4}{*}{PixArt-$\alpha$~\cite{chenpixart}} 
 &Stable Diffusion v1.4 &\multirow{4}{*}{60} &\multirow{4}{*}{4} &56 &8 &\multirow{4}{*}{0.338} &\multirow{4}{*}{0.361} &\multirow{4}{*}{0.563} \\
 &Latent Consistency Model & & &41 &23 & & & \\
 &Stable Cascade & & &59 &5 & & & \\
 &Real Data & & &52 &12 & & & \\ \cline{1-9}
 \multirow{4}{*}{Stable Cascade~\cite{perniaswurstchen}} 
 &Stable Diffusion v1.4 &\multirow{4}{*}{60} &\multirow{4}{*}{4} &53 &11 &\multirow{4}{*}{0.381} &\multirow{4}{*}{0.377} &\multirow{4}{*}{0.590} \\
 &Latent Consistency Model & & &36 &28 & & & \\
 &PixArt-$\alpha$ & & &53 &11 & & & \\
 &Real Data & & &52 &12 & & & \\ \cline{1-9}
 \textbf{Average Value} & & & & & &0.356 &0.364 &0.567 \\
  \Xhline{1.0pt}
\end{tabular}
\vspace{-1em}
\caption{The results of TrainProVe-Ent on ImageNet-100.}
\vspace{-1.5em}
\label{tab:ent_imagenet100}
\end{table*}

\begin{table*}
\small
\centering
\vspace{-3em}
\addtolength{\tabcolsep}{-1.5pt}
\begin{tabular}{ccccccccc}

\Xhline{1.0pt}
   \multicolumn{2}{c}{\textbf{The Suspect's Data Sources}} &\multirow{2}{*}{\textbf{TP}} &\multirow{2}{*}{\textbf{FP}} &\multirow{2}{*}{\textbf{FN}} &\multirow{2}{*}{\textbf{TN}} &\multirow{2}{*}{\textbf{Accuracy}} &\multirow{2}{*}{\textbf{F1 Score}} &\multirow{2}{*}{\textbf{AUROC}} \\ \cline{1-2}
 \textbf{The Defender's Generative Model} $G_d$ &\textbf{Data Sources Unrelated to} \( G_d \) & & & & & &  \\ \hline
\multirow{4}{*}{Stable Diffusion~\cite{rombach2022high} v1.4} 
 &Latent Consistency Model &\multirow{4}{*}{64} &\multirow{4}{*}{0} &0 &64 &\multirow{4}{*}{0.819} &\multirow{4}{*}{0.688} &\multirow{4}{*}{0.887} \\
 &PixArt-$\alpha$ & & &0 &64 & & & \\
 &Stable Cascade & & &11 &53 & & & \\
 &Real Data & & &47 &17 & & & \\ \cline{1-9}
\multirow{4}{*}{Latent Consistency Model~\cite{luo2023latent}} 
 &Stable Diffusion v1.4 &\multirow{4}{*}{2} &\multirow{4}{*}{62} &0 &64 &\multirow{4}{*}{0.806} &\multirow{4}{*}{0.061} &\multirow{4}{*}{0.517} \\
 &PixArt-$\alpha$ & & &0 &64 & & & \\
 &Stable Cascade & & &0 &64 & & & \\
 &Real Data & & &0 &64 & & & \\ \cline{1-9}
 \multirow{4}{*}{PixArt-$\alpha$~\cite{chenpixart}} 
 &Stable Diffusion v1.4 &\multirow{4}{*}{64} &\multirow{4}{*}{0} &29 &35 &\multirow{4}{*}{0.734} &\multirow{4}{*}{0.601} &\multirow{4}{*}{0.834} \\
 &Latent Consistency Model & & &0 &64 & & & \\
 &Stable Cascade & & &24 &40 & & & \\
 &Real Data & & &32 &32 & & & \\ \cline{1-9}
 \multirow{4}{*}{Stable Cascade~\cite{perniaswurstchen}} 
 &Stable Diffusion v1.4 &\multirow{4}{*}{59} &\multirow{4}{*}{5} &32 &32 &\multirow{4}{*}{0.784} &\multirow{4}{*}{0.631} &\multirow{4}{*}{0.836} \\
 &Latent Consistency Model & & &0 &64 & & & \\
 &PixArt-$\alpha$ & & &0 &64 & & & \\
 &Real Data & & &32 &32 & & & \\ \cline{1-9}
 \textbf{Average Value} & & & & & &0.786 &0.495 &0.769 \\
  \Xhline{1.0pt}
\end{tabular}
\vspace{-1em}
\caption{The results of TrainProVe on ImageNet-100.}
\vspace{-1.5em}
\label{tab:acc_imagenet100}
\end{table*}

\end{document}